\documentclass{article}
\PassOptionsToPackage{numbers, compress}{natbib}

\usepackage[final]{neurips_2022}

\usepackage[utf8]{inputenc} 
\usepackage[T1]{fontenc}    
\usepackage{hyperref}       
\usepackage{url}            
\usepackage{booktabs}       
\usepackage{amsfonts}       
\usepackage{nicefrac}       
\usepackage{microtype}      
\usepackage[dvipsnames]{xcolor}
\usepackage{graphicx}
\usepackage{subfigure}
\usepackage{amsthm}
\usepackage{amssymb}
\usepackage{enumitem}
\usepackage{wrapfig}
\usepackage{dsfont}
\usepackage[ruled, noend]{algorithm2e}
\usepackage{mathrsfs}
\usepackage{multicol}
\usepackage{threeparttable}
\usepackage{multirow}
\usepackage{minitoc}
\usepackage{comment}

\usepackage{amsmath,amsfonts,bm}

\def\eqref#1{equation~\ref{#1}}

\def\1{\bm{1}}

\newcommand\independent{\protect\mathpalette{\protect\independenT}{\perp}}
\def\independenT#1#2{\mathrel{\rlap{$#1#2$}\mkern2mu{#1#2}}}

\DeclareMathAlphabet{\mathsfit}{\encodingdefault}{\sfdefault}{m}{sl}
\SetMathAlphabet{\mathsfit}{bold}{\encodingdefault}{\sfdefault}{bx}{n}

\def\gA{{\mathcal{A}}}
\def\gB{{\mathcal{B}}}

\def\gE{{\mathcal{E}}}

\def\gG{{\mathcal{G}}}
\def\gH{{\mathcal{H}}}
\def\gI{{\mathcal{I}}}
\def\gJ{{\mathcal{J}}}

\def\gM{{\mathcal{M}}}

\def\gQ{{\mathcal{Q}}}

\def\gS{{\mathcal{S}}}

\def\gU{{\mathcal{U}}}
\def\gV{{\mathcal{V}}}
\def\gW{{\mathcal{W}}}

\DeclareMathOperator*{\argmax}{arg\,max}

\newtheorem{assumption}{Assumption} 
\newtheorem{definition}{Definition} 
\newtheorem{theorem}{Theorem} 
\newtheorem{proposition}{Proposition} 
\newtheorem{lemma}{Lemma} 

\newcommand{\qa}[1]{{\textcolor{RoyalBlue}{{#1}}}}
\newcommand{\highest}[1]{{\textbf{{#1}}}}

\title{Generalizing Goal-Conditioned Reinforcement Learning with Variational Causal Reasoning}

\author{%
  Wenhao Ding$^1$, Haohong Lin$^1$, Bo Li$^2$, Ding Zhao$^1$ \\
  $^1$Carnegie Mellon University\\
  $^2$University of Illinois Urbana-Champaign \\
  \texttt{\{wenhaod, haohongl\}@andrew.cmu.edu, lbo@illinois.edu, dingzhao@cmu.edu} \\
}

\begin{document}

\doparttoc 
\faketableofcontents 

\maketitle


\begin{abstract}

As a pivotal component to attaining generalizable solutions in human intelligence, reasoning provides great potential for reinforcement learning (RL) agents' generalization towards varied goals by summarizing part-to-whole arguments and discovering cause-and-effect relations.
However, how to discover and represent causalities remains a huge gap that hinders the development of causal RL. 
In this paper, we augment Goal-Conditioned RL (GCRL) with \textit{Causal Graph (CG)}, a structure built upon the relation between objects and events.
We novelly formulate the GCRL problem into variational likelihood maximization with CG as latent variables.
To optimize the derived objective, we propose a framework with theoretical performance guarantees that alternates between two steps: using interventional data to estimate the posterior of CG; using CG to learn generalizable models and interpretable policies.
Due to the lack of public benchmarks that verify generalization capability under reasoning, we design nine tasks and then empirically show the effectiveness of the proposed method against five baselines on these tasks. 
Further theoretical analysis shows that our performance improvement is attributed to the virtuous cycle of causal discovery, transition modeling, and policy training, which aligns with the experimental evidence in extensive ablation studies. Code is available on \url{https://github.com/GilgameshD/GRADER}.

\end{abstract}

\section{Introduction}

The generalizability, which enables an algorithm to handle unseen tasks, is fruitful yet challenging in multifarious decision-making domains.
Recent literature~\cite{xu2021bayesian, cranmer2006discovering, li2020ngs} reveals the critical role of reasoning in improving the generalization of reinforcement learning (RL). 
However, most off-the-shelf RL algorithms~\cite{sutton2018reinforcement} have not regarded reasoning as an indispensable accessory, thus usually suffering from data inefficiency and performance degradation due to the mismatch between training and testing settings.
To attain generalization at the testing stage, some efforts were put into incorporating domain knowledge to learn structured information, including sub-task decomposition~\cite{lu2021learning} and program generation~\cite{yang2021program, landajuela2021discovering, han2020neuro, zhao2021proto, sun2019program}, which guide the model to solve complicated tasks in an explainable way.
However, such symbolism-dominant methods heavily depend on the re-usability of sub-tasks and pre-defined grammars, which may not always be accessible in decision-making tasks.

Inspired by the close link between reasoning and the cause-and-effect relationship, causality is recently incorporated to compactly represent the aforementioned structured knowledge in RL training~\cite{gershman2017reinforcement}. 
Based on the form of causal knowledge, we divide the related works into two categories, i.e., \textit{implicit} and \textit{explicit} causation. 
With \textit{implicit} causal representation, researchers ignore the detailed causal structure. For instance, \cite{zhang2020invariant} extracts invariant features as one node that influences the reward function, while the other node consists of task-irrelevant features~\cite{tomar2021model, sontakke2021causal, sodhani2022improving, bica2021invariant}. 
This neat structure has good scalability but requires access to multiple environments that share the same invariant feature~\cite{zhang2020invariant, bica2021invariant, han2021learning}.
In contrast, one can turn to the \textit{explicit} side by estimating detailed causal structures~\cite{wangtask, volodin2020resolving, seitzer2021causal, gasse2021causal}, which uses directed graphical models to capture the causality in the environment.
A pre-request for this estimation is the object-level or event-level abstraction of the observation, which is available in most tasks and also becoming a frequently studied problem~\cite{abel2022theory, shanahan2022abstraction, abel2018state}. 
However, existing \textit{explicit} causal reasoning RL models either require the true causal graph~\cite{nair2019causal} or rely on heuristic design without theoretical guarantees~\cite{wangtask}.

In this paper, we propose \textit{\textbf{G}ene\textbf{RA}lizing by \textbf{D}iscov\textbf{ER}ing (GRADER)}, a causal reasoning method that augments the RL algorithm with data efficiency, interpretability, and generalizability.
We mainly focus on Goal-Conditioned RL (GCRL)~\cite{liu2022goal}, where different goal distributions during training and testing reflect the generalization.
We formulate the GCRL into a probabilistic inference problem~\cite{levine2018reinforcement} with a learnable causal graph as the latent variable. 
This novel formulation naturally explains the learning objective with three components -- transition model learning, planning, and causal graph discovery -- leading to an optimization framework that alternates between causal discovery and policy learning to gain generalizability.
Under some mild conditions, we prove the unique identifiability of the causal graph and the theoretical performance guarantee of the proposed framework.

To demonstrate the effectiveness of the proposed method, we conduct comprehensive experiments in environments that require strong reasoning capability. 
Specifically, we design two types of generalization settings, i.e., spuriousness and composition, and provide an example to illustrate these settings in Figure~\ref{fig:motivation}.
The evaluation results confirm the advantages of our method in two aspects.
First, the proposed data-efficient discovery method provides an explainable causal graph yet requires much fewer data than previous methods, increasing data efficiency and interpretability during task solving.
Second, simultaneously discovering the causal graph during policy learning dramatically increases the success rate of solving tasks.
In summary, the contribution of this paper is threefold:
\vspace{-2mm}
\begin{itemize}[leftmargin=0.3in]
    \item We use the causal graph as a latent variable to reformulate the GCRL problem and then derive an iterative training framework from solving this problem.
    \item We prove that our method uniquely identifies true causal graphs, and the performance of the iterative optimization is guaranteed with a lower bound given converged transition dynamics.
    \item We design nine tasks in three environments that require strong reasoning capability and show the effectiveness of the proposed method against strong baselines on these tasks.
    \vspace{-4mm}
\end{itemize}

\begin{figure}
    \centering
    \includegraphics[width=0.9\textwidth]{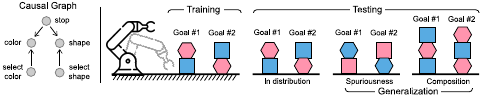}
    \caption{The robot picks and places objects to achieve given goals. \textbf{Left:} The causal graph of this task.
    \textbf{Right:} Training setting: the hexagon is always pink and the box is always blue. Three testing settings: (1) In distribution: the same as the training setting. (2) Spuriousness: swap the color and shape to break the spurious correlation. (3) Composition: increase the height of the goal.}
    \label{fig:motivation}
    \vspace{-4mm}
\end{figure}

\section{Problem Formulation and Preliminary}

We start by discussing the setting we consider in this paper and the assumptions required in causal reasoning. Then we briefly introduce the necessary concepts related to causality and causal discovery.

\subsection{Factorized Goal-conditioned RL}
\label{FactorMDP}

We assume the environment follows the Goal-conditioned Markov Decision Process (MDP) setting with full observation.
This setting is represented by a tuple $\mathcal{M}=(\mathcal{S}, \mathcal{A}, \mathcal{P}, \mathcal{R}, {G})$, where $\mathcal{S}$ is the state space, $\mathcal{A}$ is the action space, $\mathcal{P}$ is the probabilistic transition model, ${G} \subset \mathcal{S}$ is the goal space which is a set of assignment of values to states, and $r(s, g) = \mathds{1}(s = g) \in \mathcal{R}$
is the sparse deterministic reward function that returns 1 only if the state $s$ match the goal $g$.
In this paper, we focus on the goal-conditioned generalization problem, where the goal for training and testing stages will be sampled from different distributions $p_{\text{train}}(g)$ and $p_{\text{test}}(g)$. 
We refer to a goal $g \in G$ as a task and use these two terms interchangeably.
To accomplish the causal discovery methods, we make a further assumption similar to~\cite{seitzer2021causal, boutilier2000stochastic} for the state and action space:
\begin{assumption}[Space Factorization]\label{assumption_factorize}
The state space $\mathcal{S} = \{ \mathcal{S}_1 \times \cdots \times \mathcal{S}_M \}$ and action space $\mathcal{A} = \{ \mathcal{A}_1 \times \cdots \times \mathcal{A}_N \}$ can be factorized to disjoint components $\{\mathcal{S}_i\}_{i=1}^{M}$ and $\{\mathcal{A}_i\}_{i=1}^{N}$.
\end{assumption}
The components representing one event or object's property usually have explicit semantic meanings for better interpretability.
This assumption can be satisfied by state and action abstraction, which has been widely investigated in~\cite{abel2022theory, shanahan2022abstraction, abel2018state}.
Such factorization also helps deal with the high-dimensional states since it could be intractable to treat each dimension as one random variable~\cite{wangtask}.

\vspace{-2mm}
\subsection{Causal Reasoning with Graphical Models}
\label{cg_scm}

Reasoning with causality relies on specific causal structures, which are commonly represented as directed acyclic graphs (DAGs)~\cite{peters2017elements} over variables. Consider random variables $\bm{X} = (X_1,\dots, X_d)$ with index set $\bm{V}:= \{1,\dots,d \}$. A graph $\gG = (\bm{V}, \mathcal{E})$ consists of nodes $\bm{V}$ and edges $\mathcal{E} \subseteq \bm{V}^2$ with $(i, j)$ for any $i, j \in \bm{V}$.
A node $i$ is called a parent of $j$ if $e_{ij} \in \gE$ and $e_{ji} \notin \gE$. The set of parents of $j$ is denoted by $\textbf{PA}^{\gG}_j$. 
We formally discuss the graph representation of causality with two definitions:
\begin{definition}[Structural Causal Models~\cite{peters2017elements}]
A structural causal model (SCM) $\mathfrak{C}:= (\bm{S}, \bm{U})$ consists of a collection $\bm{S}$ of $d$ functions
$
    X_j := f_j(\textbf{PA}^{\gG}_j, U_j),\ j \in [d],\ 
$
where $\textbf{PA}_j \subset \{ X_1,\dots,X_d \} \backslash \{X_j \}$ are called parents of $X_j$; and a joint distribution $\bm{U} = \{ U_1,\dots,U_d \}$ over the noise variables, which are required to be jointly independent.
\label{def_scm}
\end{definition}

\begin{definition}[Causal Graph~\cite{peters2017elements}, CG]
The causal graph $\gG$ of an SCM is obtained by creating one node for each $X_j$ and drawing directed edges from each parent in $\textbf{PA}^{\gG}_j$ to $X_j$.
\end{definition}
We note that CG describes the structure of the causality, and SCM further considers the specific causation from the parents of $X_j$ to $X_j$ via $f_j$ as well as exogenous noises ${U_j}$.
To uncover the causal structure from data distribution, we assume that the CG satisfies the \textit{Markov Property} and \textit{Faithfulness}~\cite{peters2017elements}, which make the independences consistent between the joint distribution $P(X_1,\dots,X_n)$ and the graph $\gG$. 
We also follow the \textit{Causal Sufficiency} assumption~\cite{spirtes2000causation} that supposes we have measured all the common causes of the measured variables.

Existing work~\cite{pitis2020counterfactual, seitzer2021causal} believes that two objects have causality only if they are close enough while there is no edge between them if the distance is large.
Instead of using such a local view of the causality, we assume the causal graph is consistent across all time steps, which also handles the local causal influence. The specific influence indicated by edges is estimated by the function $f_j(\textbf{PA}_j, U_j)$. 

\begin{wrapfigure}{r}{0.5\textwidth}
\vspace{-5mm}
\centering
\includegraphics[width=0.5\textwidth]{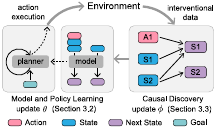}
\caption{The paradigm of GRADER.}
\label{fig:framework}
\vspace{-10mm}
\end{wrapfigure}

\section{Generalizing by Discovering (GRADER)}

With proper definitions and assumptions, we now look at the proposed method. We first derive the framework of GRADER by formulating the GCRL problem as a latent variable model, which provides a variational lower bound to optimize. Then, we divide this objective function into three parts and iteratively update them with a performance guarantee.

\subsection{GCRL as Latent Variable Models}
\label{gcrl_lvm}

The general objective of RL is to maximize the expected reward function w.r.t. a learnable policy model $\pi$. Particularly, in the goal-conditioned setting, such objective is represented as $\max_{\pi} \mathbb{E}_{\tau \sim \pi, g\sim p(g)}[ \sum_{t=0}^{T} r(s^t, g)]$, where $p(g)$ is the distribution of goal and $\tau := \{ s^0,a^0,\dots,s^{T} \}$ is the action-state trajectory with maximal time step $T$. The trajectory ends only if the goal is achieved $g = s^T$ or the maximal step is reached.

Inspired by viewing the RL problem as probabilistic inference~\cite{abdolmaleki2018maximum, levine2018reinforcement}, we replace the objective from \textit{``How to find actions to achieve the goal?''} to \textit{``what are the actions if we achieve the goal?''}, leading to a likelihood maximization problem for $p(\tau|s^*)$ with $s^*:= \mathds{1}(g=s^T)$.
Different from previous work~\cite{marino2019inference} that recasts actions as latent variables and infers actions that result in ``observed'' high reward, we decompose $p(\tau|s^*)$ with $\gG$ as the latent variable to get the evidence lower bound (ELBO)
\begin{equation}
    \log p(\tau|s^*) = \log \int p(\tau|\gG, s^*) p(\gG | s^*) d\gG  
    \geq \mathbb{E}_{q(\gG|\tau)}[\log p(\tau|\gG, s^*)] - \mathbb{D}_{\text{KL}}[q(\gG|\tau) || p(\gG)], \\
\label{elbo}
\end{equation}
where the prior $p(\gG)$ and variational posterior $q(\gG|\tau)$ represent distributions over graph structures, i.e, the probability of the existences of edges in graphs. $\mathbb{D}_{\text{KL}}$ denotes the Kullback–Leibler (KL) divergence between two graphs, which will be explained in Section~\ref{causal_discovery}.
Recall that the space of goal $G$ is a subset of the state, we extend the meaning of $g$ and assume all trajectories achieve the goal in the final state~\cite{andrychowicz2017hindsight}, i.e., $g = s^T$. Such extension makes it possible to further decompose the first term of (\ref{elbo}) as (refer to Appendix~\ref{app:derivation_transition})
\begin{equation}
    \log p(\tau|\gG, s^*) = \log p(s^0) + \sum_{t=0}^{T-1} \log p(s^{t+1}| s^{t}, a^{t}, \gG) + \sum_{t=0}^{T-1} \log \pi(a^{t}|s^{t}, s^*, \gG) + \log p(g).
\label{transition}
\end{equation}

Here, we use the fact that $\gG$ is effective in both the transition model $p(s^{t+1}|a^{t}, s^{t}, \gG)$ and the policy model $\log p(a^{t}|s^{t}, s^*, \gG)$, $g$ only influences the policy model, and the initial state $s_0$ depends neither on $\gG$ nor $g$.
We also assume that both initial state $\log p(s_0)$ and goal $\log p(g)$ follow the uniform distribution. Thus, the first and last terms of (\ref{transition}) are constants. 
The policy term $\pi$, selecting action $a^t$ according to both current state $s^t$ and goal $g$, is implemented with the planning method and is further discussed in Section~\ref{model_planning}. 
Finally, we maximize the likelihood $p(\tau|s^*)$ with the following reformulated ELBO as the objective
\begin{equation}
\begin{aligned}
    \gJ(\theta,\phi) = \mathbb{E}_{ q_{\phi}(\gG|\tau)}\sum_{t=0}^{T-1} \left[ \log p_{\theta}(s^{t+1}|s^{t}, a^{t}, \gG) +  \log \pi_{\theta}(a^{t}|s^{t}, s^{*}, \gG)\right] - \mathbb{D}_{\text{KL}}[q_{\phi}(\gG|\tau) || p(\gG)]
\label{elbo_final}
\end{aligned}
\end{equation}
where $\theta$ is the shared parameter of transition model $p_{\theta}(s^{t+1}|a^{t}, s^{t}, \gG)$ and policy $\pi_{\theta}(a^{t}|s^{t}, s^{*}, \gG)$, and $\phi$ is the parameter of causal graph $q_{\phi}(\gG|\tau)$.
To efficiently solve this optimization problem, we iteratively updates parameter $\phi$ (causal discovery, Section~\ref{causal_discovery}) and parameter $\theta$ (model and policy learning, Section~\ref{model_planning}), as shown in Figure~\ref{fig:framework}.
Intuitively, these processes can be respectively viewed as the discovery of graph and the update of $f_i$, which share tight connections as discussed in Section~\ref{cg_scm}. 

\subsection{Model and Policy Learning}
\label{model_planning}

Let us start with a simple case where we already obtain a $\gG$ and use it to guide the learning of parameter $\theta$ via $\max_{\theta} \gJ(\theta,\phi)$.
Since the KL divergence of $\gJ(\theta,\phi)$ does not involve $\theta$, we only need to deal with the first expectation term, i.e., the likelihood of transition model and policy.
For the transition $p_{\theta}(s^{t+1}|a^{t}, s^{t}, \gG)$, we connect it with causal structure by further defining a particular type of CG and denote it as $\gG$ in the rest of this paper:
\begin{definition}[Transition Causal Graph]
\label{transition_cg}
We define a bipartite graph $\gG$, whose vertices are divided into two disjoint sets $\mathcal{U}= \{\mathcal{A}^{t}, \mathcal{S}^{t} \}$ and $\mathcal{V}=\{ \mathcal{S}^{t+1} \}$. $\mathcal{A}^{t}$ represents action nodes at step $t$, $\mathcal{S}^t$ state nodes at step $t$, and $\mathcal{S}^{t+1}$ the state nodes at step $t+1$. All edges start from set $\mathcal{U}$ and end in set $\mathcal{V}$.
\end{definition}

\textbf{Model learning.} This definition builds the causal graph between two consecutive time steps, which indicates that the values of states in step $t+1$ depend on values in step $t$. It also implies that the interventions~\cite{peters2017elements} on nodes in $\mathcal{U}$ are directly obtained since they have no parent nodes. 
We denote the marginal distribution of $\mathcal{S}$ as $p_{\gI_\pi^s}$, which is collected by RL policy $\pi$.
Combined with the Definition~\ref{def_scm} of SCM, we find that $p_{\theta}(s^{t+1}|a^{t}, s^{t}, \gG)$ essentially approximates a collection of functions $f_j$ following the structure $\gG$, which take as input the values of parents of the state node $s_j$ and outputs the value $s_j$.
Thus, we propose to model the transition corresponding to $\gG$ with a collection of neural networks $f_{\theta}(\gG) := \{ f_{\theta_j} \}_{j=1}^{M}$ to obtain
\begin{equation}
    {s}^{t+1}_j = f_{\theta_j}([ \textbf{PA}_j^{\gG}]^{t}, U_j),
\label{gru_inference}
\end{equation}
where $[\textbf{PA}_j^{\gG}]^{t}$ represents the values of all parents of node $s^{t}_j$ at time step $t$ and $U_j$ follows Gaussian noise $U_j \sim \mathcal{N}(0, \textbf{I})$.
In practice, we use Gated Recurrent Unit~\cite{chung2014empirical} as $f_j$ because it supports varying numbers of input nodes. 
We take $s^{t}_j$ as the initial hidden embedding and the rest parents $[ \textbf{PA}_j^{\gG}\backslash s_j ]^{t}$ as the input sequence to $f_j$.
The entire model is optimized by stochastic gradient descent with the log-likelihood $\log p_{\theta}(s^{t+1}|a^{t}, s^{t}, \gG)$ as objective.

\textbf{Policy learning with planning.}
Then we turn to the policy term $\pi_{\theta}(a^{t}|s^{t}, s^{*}, \gG)$ in $\gJ(\theta,\phi)$. 
We optimize it with planning methods that leverage the estimated transition model.
Specifically, the policy aims to optimize an action-state value function $Q(s^t, a^t)=\mathbb{E}\left[\sum_{t'=0}^H \gamma^{t'} r(s^{t'+t}, a^{t'+t}) | s^t, a^t\right]$, which can be obtained by unrolling the transition model with a horizon of $H$ steps and discount factor $\gamma$.
In practice, we use model predictive control (MPC)~\cite{camacho2013model} with random shooting~\cite{richards2005robust}, which selects the first action in the fixed-horizon trajectory that has the highest action-state value $Q(s^t, a^t)$, i.e. $\hat{\pi}(s^t) = \argmax_{a^t\in \gA} Q_\theta^\gG (s^t, a^t)$. 
The formulation we derived so far is highly correlated to the model-based RL framework~\cite{wang2019benchmarking}. However, the main difference is that we obtain it with variational inference by regarding the causal graph as a latent variable.



\vspace{-3mm}
\subsection{Data-Efficient Causal Discovery}
\label{causal_discovery}

\begin{wrapfigure}{r}{0.43\textwidth} 
\vspace{-5mm}
\newcommand\mycommfont[1]{\small\ttfamily\textcolor{RoyalBlue}{#1}}
\SetCommentSty{mycommfont}
\begin{algorithm}[H]
\caption{GRADER Training}
\label{algorithm1}
\KwIn{Trajectory buffer $\mathcal{B}_{\tau}$, Causal graph $\gG$, Transition model $f_{\theta}$, causal discovery threshold $\eta$}
\While{$\theta$ not converged}
{
\tcp{Policy from planning} 
Sample a goal $g\sim p_{\text{train}}(g)$ \\
\While{$t < T$}
{
    $a^t \leftarrow$ Planner($f_{\theta}$, $s^t$, $g$) \\
    $s^{t+1}$, $r^t \leftarrow$ Env($a^t$, $g$) \\
    $\mathcal{B}_{\tau} \leftarrow \mathcal{B}_{\tau}\cup\{a^t, s^t, s^{t+1} \} $
}
\tcp{Estimate causal graph} 
\For{$i \leq M + N$}
{
    \For{$j \leq M$}
    {
        Infer edge $e_{ij} \leftarrow q_{\phi}(\cdot|\mathcal{B}, \eta)$
    }
}
\tcp{Learn transition model} 
Update $f_{\theta}(\gG)$ via (\ref{gru_inference}) with $\mathcal{B}$ \\
}
\end{algorithm}
\vspace{-6mm}
\end{wrapfigure}

In this step, we relax the assumption of knowing $\gG$ and aim to estimate the posterior distribution $q(\gG|\tau)$ to optimize ELBO~(\ref{elbo_final}) w.r.t. parameter $\phi$.
In most score-based methods~\cite{chickering2002optimal}, likelihood is used to evaluate the correctness of the causal graph, i.e., a better causal graph leads to a higher likelihood.
Since the first term of (\ref{elbo_final}) represents the likelihood of the transition model, we convert the problem of $\max_{\phi} \gJ(\theta, \phi)$ to the causal discovery that finds the true causal graph based on collected data samples.
As for the second term of (\ref{elbo_final}), the following proposition shows that the KL divergence between $q_{\phi}(\gG|\tau)$ and $p(\gG)$ can be approximated by a sparsity regularization~(proof in Appendix~\ref{app:approx_kld}).

\begin{proposition}[KL Divergence as Sparsity Regularization]
With entry-wise independent Bernoulli prior $p(\gG)$ and point mass variational distribution $q(\gG|\tau)$ of DAGs, $\mathbb{D}_{\text{KL}}[q_{\phi} \| p]$ is equivalent to an $\ell_1$ sparsity regularization for the discovered causal graph.
\label{approx_kld}
\end{proposition}

We restrict the posterior $q_{\phi}(\gG|\tau)$ to point mass distribution and use a threshold $\eta$ to control the sparsity.
We perform the discovery process from the classification perspective by proposing binary classifiers $q_{\phi}(e_{ij}|\tau, \eta)$ to determine the existence of an edge $e_{ij}$.
This classifier $q_{\phi}(e_{ij}|\tau, \eta)$ is implemented by statistic \textit{Independent Test}~\cite{chalupka2018fast} and $\eta$ is the threshold for the p-value of the hypothesis. A larger $\eta$ corresponds to harder sparsity constraints, leading to a sparse $\gG$ since two nodes are more likely to be considered independent.
According to the definition~\ref{transition_cg}, we only need to conduct classification to edges connecting nodes between $\mathcal{U}$ and $\mathcal{V}$. 
If two nodes are dependent, we add one edge directed from the node in $\mathcal{U}$ to the node in $\mathcal{V}$.
This definition also ensures that we always have $q(\gG|\tau)\in \gQ_{DAG}$, where $\gQ_{DAG}$ is the class of DAG. 
With this procedure, we identify a unique CG $\gG^*$ under optimality: 

\begin{proposition}[Identifiability]
\label{uniqueness}
Given an oracle independent test, with an optimal interventional data distribution $p_{\gI_\pi^s}^*$, causal discovery obtains $\phi^*$ that correctly tells the independence between any two nodes, then the causal graph is uniquely identifiable, with $e^{*}_{ij} = q_{\phi^*}(e_{ij}|\tau), \forall i \in [M+N], j \in [M]$.
\end{proposition}

In practice, we use $\chi^2$-test for discrete variables and the Fast Conditional Independent Test~\cite{chalupka2018fast} for continuous variables.
The sample size needed for obtaining the oracle test has been widely investigated~\cite{canonne2018testing}. However, testing with finite data is not a trivial problem, as stated in~\cite{shah2020hardness}, especially when the data is sampled from a Goal-conditioned MDP.
Usually, the random policy is not enough to satisfy the oracle assumption because some nodes cannot be fully explored when the task is complicated and has a long horizon. 
To make this assumption empirically possible, it is necessary to simultaneously optimize $\pi_{\theta}(a^t | s^t, s^*, \gG)$ to access more samples close to finishing the task, which is further analyzed in Section~\ref{sec::convergence}.
We also empirically support this argument in Section~\ref{discovery_results} and provide a detailed theoretical proof in Appendix~\ref{app:planning} and~\ref{app:causal_discovery}.

\subsection{Analysis of Performance Guarantee}
\label{sec::convergence}

The entire pipeline of GRADER is summarized in Algorithm~\ref{algorithm1}. To analyze the performance of the optimization of (\ref{elbo_final}), we first list important lemmas that connect previous steps and then show that the iteration of these steps in \textit{GRADER} leads to a virtuous cycle.

By the following lemmas, we show the following performance guarantees step by step. Lemma~\ref{monotonicity} shows model learning is monotonically better at convergence given a better causal graph from causal discovery. 
Then the learned transition model helps to improve the lower bound of the value function during planning according to Lemma~\ref{Value Bound}.
Lemma~\ref{TV_distance} reveals the connection between policy learning and interventional data distribution, which in turn improves the quality of our causal discovery, as is shown in Lemma~\ref{causal_discovery_quality} and Proposition~\ref{uniqueness}.

\begin{lemma}[Monotonicity of Transition Likelihood]
\label{monotonicity}
Assume $\gG^*=(V,E^*)$ be the true CG, for two CG $\gG_1=(V, E_1)$ and $\gG_2=(V, E_2)$, 
if $\text{SHD}(\gG_1, \gG^*) < \text{SHD}(\gG_2, \gG^*)$, $\exists \ e,\  s.t.\  E_1\cup \{e\} = E_2$, when transition model $\theta$ converges, the following inequality holds for the transition model in (\ref{elbo_final}):
\begin{equation}
    \log p_\theta(s^{t+1}|a^{t}, s^{t}, \gG^*) \geq \log p_\theta(s^{t+1}|a^{t}, s^{t}, \gG_1) \geq \log p_\theta(s^{t+1}|a^{t}, s^{t}, \gG_2)
\label{estep_eq2}
\end{equation}
where SHD is the Structural Hamming Distance defined in Appendix~\ref{proof_monotonicity}. 
\end{lemma}

\begin{lemma}[Bounded Value Function in Policy Learning]
\label{Value Bound}    
Given a planning horizon $H\to \infty$, if we already have an approximate transition model $\mathbb{D}_{\text{TV}}(\hat{p}(s'|s,a), p(s'|s,a)) \leq \epsilon_m$, the approximate policy $\hat{\pi}$ achieves a near-optimal value function (refer to Appendix~\ref{app:planning} for detailed analysis):
\begin{equation}
    \|V^{\pi^*}(s)-V^{\hat{\pi}}(s)\|_\infty \leq \frac{\gamma}{(1-\gamma)^2} \epsilon_m
\end{equation}%
\end{lemma}

\begin{lemma}[Policy Learning Improves Interventional Data Distribution]
\label{TV_distance}
With a step reward $r(s, a) = \mathbb{p_g} \mathds{1}(s=g)$, we show that the value function determines an upper bound for TV divergence between the interventional distribution with its optimal (proof details in Appendix~\ref{app:planning}):
\begin{equation}
    \begin{aligned}
    \mathbb{D}_{\text{TV}}(p_{\gI^s_\pi}, p_g) \leq 1 - (1-\gamma) V^\pi (s).
    \end{aligned}
\end{equation}
where $p_{\gI_\pi^s}$ is the marginal state distribution in interventional data, and $p_g$ is the goal distribution.
A better policy with larger $V^\pi(s)$ enforces the distribution of interventional data toward the goal.
\end{lemma}

\begin{lemma}[Interventional Data Benefits Causal Discovery]
\label{causal_discovery_quality}
For $\epsilon_g=\min_{p_g>0} p_g$, $\mathbb{D}_{\text{TV}}(p_{\gI_\pi^s}, p_g)< \epsilon_g$, the error of our causal discovery is upper bounded with $\mathbb{E}_{\hat{\gG}}[ \text{SHD}(\hat{\gG}, \gG^*)] \leq |\gS|-1$.
\end{lemma}

After the close-loop analysis of our model, we are now able to analyze the overall performance of the proposed framework.
Under the construction of $p_\theta(s^{t+1}|a^t, s^t, \gG)$ with NN-parameterized functions, the following theorem shows that the learning process will guarantee to perform a close estimation of true ELBO under the iterative optimization among model learning, planning, and causal discovery.

\begin{theorem}
\label{convergence}
With T-step approximate transition dynamics $\mathbb{D}_{\text{TV}}\Big(\hat{p}(s'|s,a), p(s'|s,a)\Big) \leq \epsilon_m$, if the goal distribution satisfies $\epsilon_g > \frac{\gamma}{1-\gamma} \epsilon_m$, and the distribution prior CG is entry-wise independent $\text{Bernoulli}(\epsilon_\gG)$, GRADER guarantees to achieve an approximate ELBO $\hat{\gJ}$ with the true ELBO $\gJ^*$:
\begin{equation}
    \|\gJ^*(\theta, \phi) - \hat{\gJ}(\hat{\theta}, \hat{\phi}) \|_{\infty} \leq \left[1 + \frac{\gamma}{(1-\gamma)^2}\right] \epsilon_m  T+ \log \left(\frac{1-\epsilon_\gG}{\epsilon_\gG} \right) (|\gS|-1) ,
\end{equation}
\end{theorem}

An intuitive understanding of the performance guarantee is that a better transition model indicates a better approximation of objective $\gJ$. 
The proof of this theorem and corresponding empirical results are in Appendix~\ref{app:convergence}.

\section{Experiments}

In this section, we first discuss the setting of our designed environments as well as the baselines used in the experiments. Then, we provide the numerical results and detailed discussions to answer the following important research questions: 
\qa{\textbf{Q1.}} Compared to existing strong baselines, how does GRADER gain performance improvement under both in-distribution and generalization settings? 
\qa{\textbf{Q2.}} Compared to an offline random policy, how does a well-trained policy improve the results of causal discovery? 
\qa{\textbf{Q3.}} Compared to score-based causal discovery, does the proposed data-efficient causal discovery pipeline guarantee identifying the true causal graph as stated in Section~\ref{causal_discovery}?
\qa{\textbf{Q4.}} Considering the correctness of causal graphs, how does the imperfect causal graph influence the task-solving performance of GCRL agents?

\begin{wrapfigure}{r}{0.3\textwidth}
\vspace{-5mm}
\centering
\includegraphics[width=0.3\textwidth]{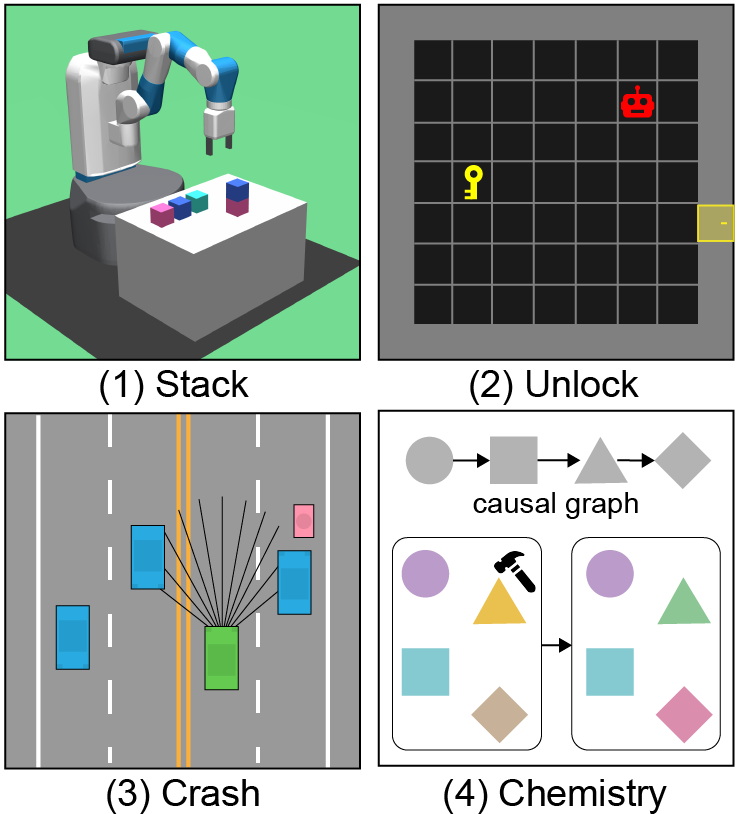}
\vspace{-7mm}
\caption{Environments.}
\label{fig:env}
\vspace{-6mm}
\end{wrapfigure}

\subsection{Environments and Baselines}

Since most commonly used RL benchmarks do not explicitly require causal reasoning for generalization, we design three new environments, which are shown in Figure~\ref{fig:env} (excluding \textit{Chemistry}~\cite{ke2021systematic}).
These environments use the true state as observation to disentangle the reasoning task from visual understanding. 
For each environment, we design three settings -- in-distribution (\textit{I}), spuriousness (\textit{S}), and composition (\textit{C}) -- corresponding to different goal distributions for generalization.
We use $p_{\text{train}}(g)$ and $p_{\text{test}}(g)$ to represent the goal distribution during training and testing, respectively.
$I$ uses the same $p_{\text{train}}(g)$ and $p_{\text{test}}(g)$, $S$ introduces spurious correlations in $p_{\text{train}}(g)$ but remove them in $p_{\text{test}}(g)$, and $C$ contains more similar sub-goals in $p_{\text{test}}(g)$ than in $p_{\text{train}}(g)$.
The details of these settings in are briefly summarized in the following (details in Appendix~\ref{detail_env}):
\begin{itemize}[leftmargin=0.2in]
    \item \textit{Stack}: 
    We design this manipulation task inspired by the CausalWorld~\cite{ahmed2020causalworld}, where the agent must stack objects to match specific shapes and colors.
    In \textit{Stack-S}, we let the same shape have the same color in $p_{\text{train}}(g)$ but randomly sample the color and shape in $p_{\text{test}}(g)$.
    In \textit{Stack-C}, the maximum number of object is two in $p_{\text{train}}(g)$ but five in $p_{\text{test}}(g)$.
    \item \textit{Unlock}: 
    We design this indoor house-holding task for the agent to collect a key to open doors. This environment is built upon the Minigrid~\cite{gym_minigrid}.
    In \textit{Stack-S}, the door and the key are always in the same row in $p_{\text{train}}(g)$ but uniformly sample in $p_{\text{test}}(g)$.
    In \textit{Unlock-C}, there are one door in $p_{\text{train}}(g)$ but two doors in $p_{\text{test}}(g)$.
    \item \textit{Crash}: 
    The occurrence of accidents usually relies on causality, e.g., an autonomous vehicle (AV) collides with a jaywalker because its view is blocked by another car~\cite{tavares2021language}.
    We design such a crash scenario based on highway-env~\cite{highway-env}, where the goals are to create crashes between a pedestrian and different AVs.
    In \textit{Stack-S}, the initial distance between AV and pedestrian is a constant in $p_{\text{train}}(g)$ but irrelevant in $p_{\text{test}}(g)$.
    In \textit{Stack-C}, there is one pedestrian in $p_{\text{train}}(g)$ but two in $p_{\text{test}}(g)$.
    \item \textit{Chemistry}~\cite{ke2021systematic}: There are 10 nodes with different colors. An underlying causal graph controls the color-changing mechanism of all nodes. In one step, the agent changes the color of one node. The goal is to match the given colors of all nodes. In the spuriousness setting, we let all nodes have the same target color. There is no composition setting in this environment.
\end{itemize}

We use the following methods as our baselines to fairly demonstrate the advantages of GRADER.
\textbf{SAC:}~\cite{haarnoja2018soft} Soft Actor-Critic is a well-known model-free RL method that uses entropy to increase the diversity of action.
\textbf{ICIN:}~\cite{nair2019causal} It uses DAgger~\cite{ross2011no} to learn goal-conditioned policy with the causal graph estimated from the expert policy. We assume it can access the true causal graph for supervised learning.
\textbf{PETS:}~\cite{chua2018deep} We consider the ensemble transition model with random shoot planning as one baseline, which achieves generalization with the uncertainty-aware design.
\textbf{TICSA:}~\cite{wangtask} This is a causal-augmented MBRL method that simultaneously optimizes a soft adjacent matrix (representing the causality) and a transition model.
\textbf{ICIL:}~\cite{bica2021invariant} This method proposes an invariant feature learning structure that captures the implicit causality of multiple tasks. We only use it for transition model learning since the original method is designed for imitation learning.
\textbf{GNN:}~\cite{schlichtkrull2018modeling} Since graph neural networks are good at learning structural information, we implement a GNN-based baseline using Relational Graph Convolutional Network.

\begin{table}[t]
\caption{Success rate (\%) for nine settings in three environments. \highest{Bold} font means the best.}
\label{tab:overall}
\centering
\small{
  \begin{tabular}{l|p{0.9cm} p{0.9cm} p{1.1cm} | p{0.9cm} p{0.95cm} p{1.1cm} | p{0.9cm} p{0.9cm} p{1.0cm}}
  \toprule
\small{Method} & \scriptsize{Stack-I} & \scriptsize{Stack-S} & \scriptsize{Stack-C} & \scriptsize{Unlock-I} & \scriptsize{Unlock-S} & \scriptsize{Unlock-C} & \scriptsize{Crash-I} & \scriptsize{Crash-S} & \scriptsize{Crash-C}  \\
\midrule
\small{SAC} & 34.7\scriptsize{$\pm$16.1} & 22.1\scriptsize{$\pm$14.0} & 31.7\scriptsize{$\pm$5.1} & 0.1\scriptsize{$\pm$0.5} & 0.0\scriptsize{$\pm$0.2} & 0.4\scriptsize{$\pm$1.7} & 22.5\scriptsize{$\pm$17.6} & 18.6\scriptsize{$\pm$8.7} & 6.7\scriptsize{$\pm$3.8}  \\
\small{ICIN} & 71.8\scriptsize{$\pm$6.9} & 71.0\scriptsize{$\pm$7.4} & 58.6\scriptsize{$\pm$8.3} & 31.7\scriptsize{$\pm$9.6} & 32.7\scriptsize{$\pm$8.6} & 31.5\scriptsize{$\pm$8.5} & 27.9\scriptsize{$\pm$6.1} & 15.8\scriptsize{$\pm$17.2} & 7.8\scriptsize{$\pm$8.8}  \\
\small{PETS} & \highest{97.2\scriptsize{$\pm$6.9}} & 77.7\scriptsize{$\pm$13.5} & 73.7\scriptsize{$\pm$10.3} & 59.5\scriptsize{$\pm$7.2} & 20.6\scriptsize{$\pm$5.9} & 28.3\scriptsize{$\pm$10.0} & 52.3\scriptsize{$\pm$11.5} & 44.6\scriptsize{$\pm$12.5} & 37.1\scriptsize{$\pm$5.1}  \\
\small{TICSA} & 85.9\scriptsize{$\pm$8.4} & 88.8\scriptsize{$\pm$10.1} & 76.2\scriptsize{$\pm$8.3} & 58.5\scriptsize{$\pm$12.3} & 33.6\scriptsize{$\pm$14.3} & 29.8\scriptsize{$\pm$8.3} & 68.9\scriptsize{$\pm$5.9} & 56.8\scriptsize{$\pm$8.6} & 15.0\scriptsize{$\pm$8.2}  \\
\small{ICIL} & 93.7\scriptsize{$\pm$5.9} & 81.2\scriptsize{$\pm$14.4} & 62.8\scriptsize{$\pm$13.0} & \highest{67.1\scriptsize{$\pm$11.6}} & 15.9\scriptsize{$\pm$4.7} & 53.6\scriptsize{$\pm$15.3} & 55.3\scriptsize{$\pm$20.9} & 21.7\scriptsize{$\pm$17.7} & 14.3\scriptsize{$\pm$7.3}  \\
{\small{GNN}} & {45.7\scriptsize{$\pm$9.1}} & {39.0\scriptsize{$\pm$10.4}} & {41.7\scriptsize{$\pm$8.6}} & {3.4\scriptsize{$\pm$2.3}} & {3.4\scriptsize{$\pm$2.4}} & {4.5\scriptsize{$\pm$3.0}} & {4.2\scriptsize{$\pm$4.0}} & {5.1\scriptsize{$\pm$5.1}} & {3.8\scriptsize{$\pm$2.8}}  \\
\midrule
\small{Score} & 92.7\scriptsize{$\pm$7.4} & 90.5\scriptsize{$\pm$7.5} & 73.9\scriptsize{$\pm$8.5} & 44.9\scriptsize{$\pm$28.1} & 23.1\scriptsize{$\pm$7.6} & 36.2\scriptsize{$\pm$30.1} & 42.3\scriptsize{$\pm$17.5} & 53.4\scriptsize{$\pm$18.7} & 8.4\scriptsize{$\pm$6.1}  \\
\small{Full} & 92.9\scriptsize{$\pm$6.3} & 86.0\scriptsize{$\pm$9.5} & 75.7\scriptsize{$\pm$10.3} & 63.8\scriptsize{$\pm$9.2} & 18.3\scriptsize{$\pm$7.4} & 53.7\scriptsize{$\pm$14.3} & 69.8\scriptsize{$\pm$14.0} & 52.6\scriptsize{$\pm$12.8} & 42.0\scriptsize{$\pm$17.2}  \\
\small{Offline} & 96.8\scriptsize{$\pm$5.8} & 95.4\scriptsize{$\pm$6.1} & 81.4\scriptsize{$\pm$7.8} & 13.8\scriptsize{$\pm$8.1} & 13.9\scriptsize{$\pm$7.5} & 11.7\scriptsize{$\pm$6.9} & 13.1\scriptsize{$\pm$16.2} & 30.2\scriptsize{$\pm$16.5} & 14.9\scriptsize{$\pm$12.4}  \\
\scriptsize{GRADER} & 95.6\scriptsize{$\pm$5.4} & \highest{97.6\scriptsize{$\pm$6.0}} & \highest{93.7\scriptsize{$\pm$8.4}} & 64.2\scriptsize{$\pm$9.1} & \highest{61.4\scriptsize{$\pm$4.4}} & \highest{82.1\scriptsize{$\pm$9.2}} & \highest{91.5\scriptsize{$\pm$4.4}} & \highest{84.3\scriptsize{$\pm$10.0}} & \highest{84.7\scriptsize{$\pm$7.3}}  \\
\bottomrule
\end{tabular}
}
\end{table}

\subsection{Results Discussion}
\label{discovery_results}

\textbf{Overall Performance (\qa{Q1})}
We compare the testing reward of all methods under nine tasks and summarize the results in Table~\ref{tab:overall} to demonstrate the overall performance. 
Generally, our method outperforms baselines in all tasks except \textit{Stack-I} because this task is too simple for all methods.
We note that the gap between our method and baselines in \textit{S} and \textit{C} settings is more significant than in the \textit{I} setting, showing that our method still works well in the non-trivial generalization task.
As a model-free method, SAC fails in all three tasks of \textit{Unlock} and \textit{Crash} environments since they have very sparse rewards.
Without learning the causal structure of the environment, PETS even cannot fully solve \textit{Unlock-S}, \textit{Unlock-C}, and all \textit{Crash} tasks.
Both TICSA and ICIL learn the causality underlying the task so that they are relatively better than SAC and PETS. However, they are still worse than GRADER in two generalization settings because of the unstable and inefficient causal reasoning mechanism.
We also find that even if ICIN is given the true causal graph, the policy learning part cannot efficiently leverage the causality, leading to worse performance in generalization settings.

To further analyze the tendency of learning, we plot the curves of all methods under \textit{Crash} in Figure~\ref{fig:discover_graph}. Our method quickly learns to solve tasks at the beginning of the training, demonstrating high data efficiency. GRADER also outperforms other methods with large gaps in the later training phase.
The training figures of the other two environments can be found in Appendix~\ref{app:more_results}.

\begin{figure}[t]
    \centering
    \includegraphics[width=1.0\textwidth]{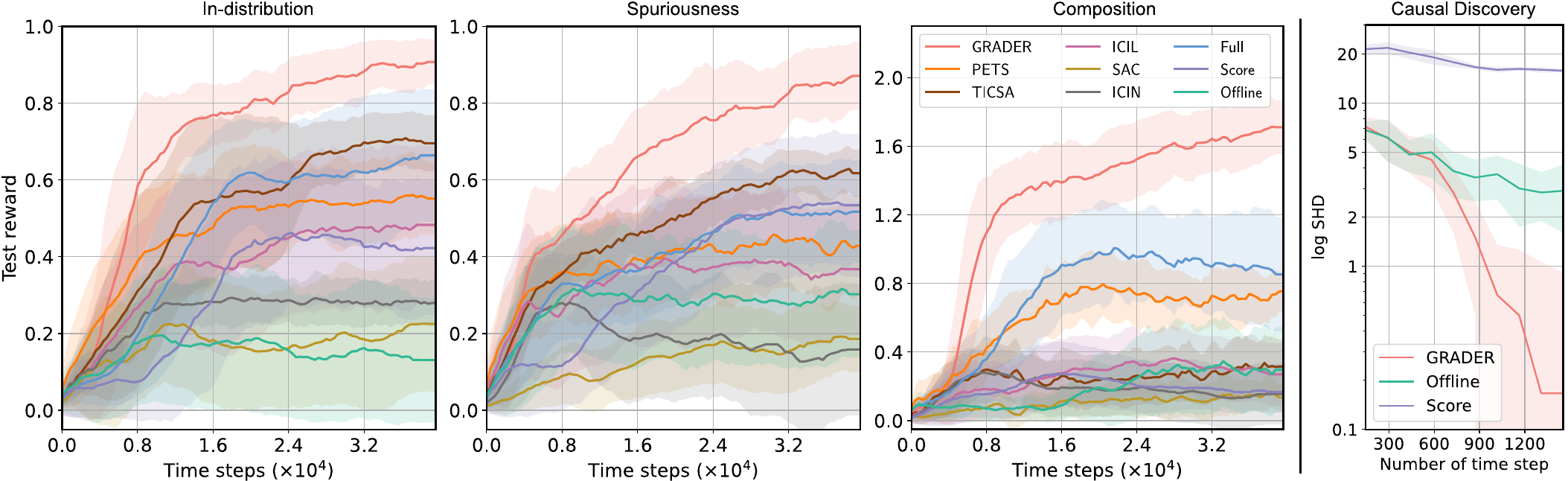}
    \caption{\textbf{Left:} Test reward of the \textit{Crash} environment calculated with 30 trails. \textbf{Right:} The accuracy of causal graph discovery with samples from GRADER, Score, and Offline.}
    \label{fig:discover_graph}
\vspace{-5mm}
\end{figure}

\begin{wrapfigure}{r}{0.35\textwidth}
\vspace{-4mm}
    \centering
    \includegraphics[width=0.35\textwidth]{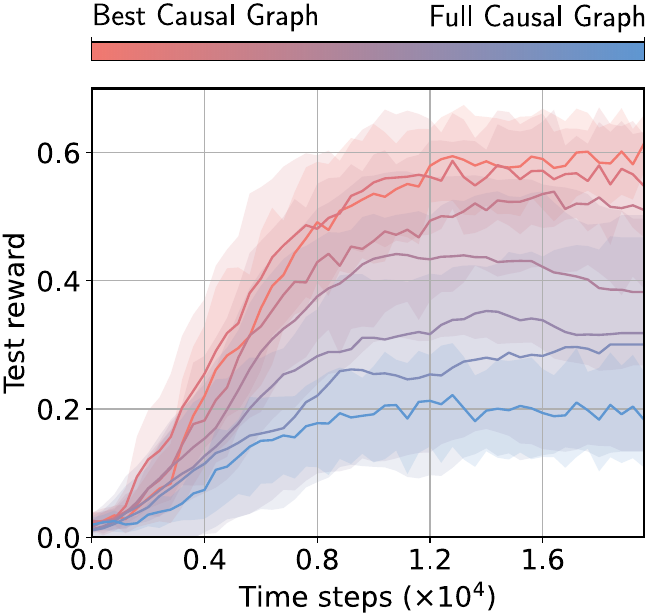}
    \caption{Influence of different causal graphs in \textit{Unlock-S}.}
    \label{fig:influence}
\vspace{-5mm}
\end{wrapfigure}

\textbf{Importance of Policy Learning (\qa{Q2})} As we mentioned in Section~\ref{causal_discovery}, we empirically compare GRADER and \textbf{Offline}~\cite{zhu2022offline}, which uses data collected from offline random policy, and plot results in the right part of Figure~\ref{fig:discover_graph}.
We use SHD~\cite{tsamardinos2006max} to compute the distance between the estimated causal graph and the true causal graph. 
The true causal graph for each environment can be found in Appendix~\ref{app:causal_graph}.
When we only use samples collected offline by random policy, we cannot obtain variables' values that require long-horizon reasoning, e.g., the door can be opened only if the agent is close to the door and has the key. 
As a consequence, the causal graph obtained by Offline harms the performance, as shown in Figure~\ref{fig:discover_graph}.
Instead, GRADER gradually explores more regions and quickly obtains the true causal graph when we iteratively discover the causal graph and update the policy.

\textbf{Advantage of Data-efficient Causal Discovery (\qa{Q3})} To show the advantage of proposed constraint-based methods, we design a model named \textbf{Score} that optimizes a soft adjacent matrix using score-based method~\cite{brouillard2020differentiable}, which is recently combined with NN for differentiable training, for example, in TICSA.
According to the discovery accuracy shown in the right part of Figure~\ref{fig:discover_graph}, we find that score-based discovery is inefficient. Based on the performance of the Score model summarized in Table~\ref{tab:overall}, we also conclude that it is not as good as our constraint-based method and has a large variance due to the unstable learning of the causal graph.

\textbf{Influence of Causal Graph (\qa{Q4})}
To illustrate the importance of the causal graph, we implement another variant of GRADER named \textbf{Full}, which uses a fixed full graph that connects all nodes between the sets $\mathcal{U}$ and $\mathcal{V}$.
According to the performance shown in Table~\ref{tab:overall} and Figure~\ref{fig:discover_graph}, we find that the full graph achieves worse results than GRADER because of the redundant and spurious correlation.
Intuitively, unrelated information causes additional noises to the learning procedure, and the spurious correlation creates a shortcut that makes the model extract wrong features, leading to worse results in the spuriousness generalization as shown in Table~\ref{tab:overall}.

We then investigate how the correctness of the causal graph influences the performance. 
We use fixed graphs interpolating from the best causal graph to the full graph to train a GRADER model in \textit{Unlock-S} and summarize the results in Figure~\ref{fig:influence}
The more correct the graph is, the higher reward the agent obtains, which supports our statements in Section~\ref{sec::convergence} that the causal graph is important for the reasoning tasks -- a better causal graph helps the model have better task-solving performance.

\begin{table}[t]
\caption{Discovery results on \textit{Chemistry} environment (GRADER / Score). \highest{Bold} font means the best.}
\label{tab:chemistry_discovery}
\centering
\small{
  \begin{tabular}{l|c|c|c|c}
    \toprule
    \small{Metric} &  \small{Collider}  & \small{Chain} & \small{Jungle} & \small{Full} \\
    \midrule 
    \small{SHD ($\downarrow$)}      & \textbf{3.70\scriptsize{$\pm$1.79}}\small{/15.4}\scriptsize{$\pm$7.03}  
                                    & \textbf{2.80\scriptsize{$\pm$1.83}}\small{/14.0}\scriptsize{$\pm$1.18} 
                                    & \textbf{7.00\scriptsize{$\pm$2.19}}\small{/13.8}\scriptsize{$\pm$0.40}  
                                    & \textbf{2.40\scriptsize{$\pm$1.20}}\small{/11.0}\scriptsize{$\pm$5.31} \\
    \small{Accuracy ($\uparrow$)}   & \textbf{0.99\scriptsize{$\pm$0.00}}\small{/0.87}\scriptsize{$\pm$0.06}
                                    & \textbf{0.99\scriptsize{$\pm$0.00}}\small{/0.88}\scriptsize{$\pm$0.01} 
                                    & \textbf{0.98\scriptsize{$\pm$0.00}}\small{/0.89}\scriptsize{$\pm$0.00} 
                                    & \textbf{0.99\scriptsize{$\pm$0.00}}\small{/0.91}\scriptsize{$\pm$0.09} \\
    \small{Precision ($\uparrow$)}  & \textbf{0.90\scriptsize{$\pm$0.05}}\small{/0.73}\scriptsize{$\pm$0.10} 
                                    & \textbf{0.94\scriptsize{$\pm$0.04}}\small{/0.79}\scriptsize{$\pm$0.03} 
                                    & 0.86\scriptsize{$\pm$0.04}\textbf{\small{/0.88}\scriptsize{$\pm$0.01}} 
                                    & 1.00\scriptsize{$\pm$0.00}\small{/1.00}\scriptsize{$\pm$0.00} \\
    \small{Recall ($\uparrow$)}     & \textbf{0.99\scriptsize{$\pm$0.02}}\small{/0.83}\scriptsize{$\pm$0.07} 
                                    & \textbf{0.96\scriptsize{$\pm$0.03}}\small{/0.73}\scriptsize{$\pm$0.00} 
                                    & \textbf{0.96\scriptsize{$\pm$0.02}}\small{/0.73}\scriptsize{$\pm$0.00} 
                                    & \textbf{0.96\scriptsize{$\pm$0.02}}\small{/0.83}\scriptsize{$\pm$0.17} \\
    \small{F-score ($\uparrow$)}    & \textbf{0.94\scriptsize{$\pm$0.03}}\small{/0.77}\scriptsize{$\pm$0.06} 
                                    & \textbf{0.95\scriptsize{$\pm$0.03}}\small{/0.76}\scriptsize{$\pm$0.02} 
                                    & \textbf{0.91\scriptsize{$\pm$0.03}}\small{/0.80}\scriptsize{$\pm$0.00} 
                                    & \textbf{0.98\scriptsize{$\pm$0.01}}\small{/0.90}\scriptsize{$\pm$0.10} \\
    \bottomrule
  \end{tabular}
}
\end{table}

\begin{wrapfigure}{r}{0.5\textwidth}
\vspace{-8mm}
    \centering
    \includegraphics[width=0.5\textwidth]{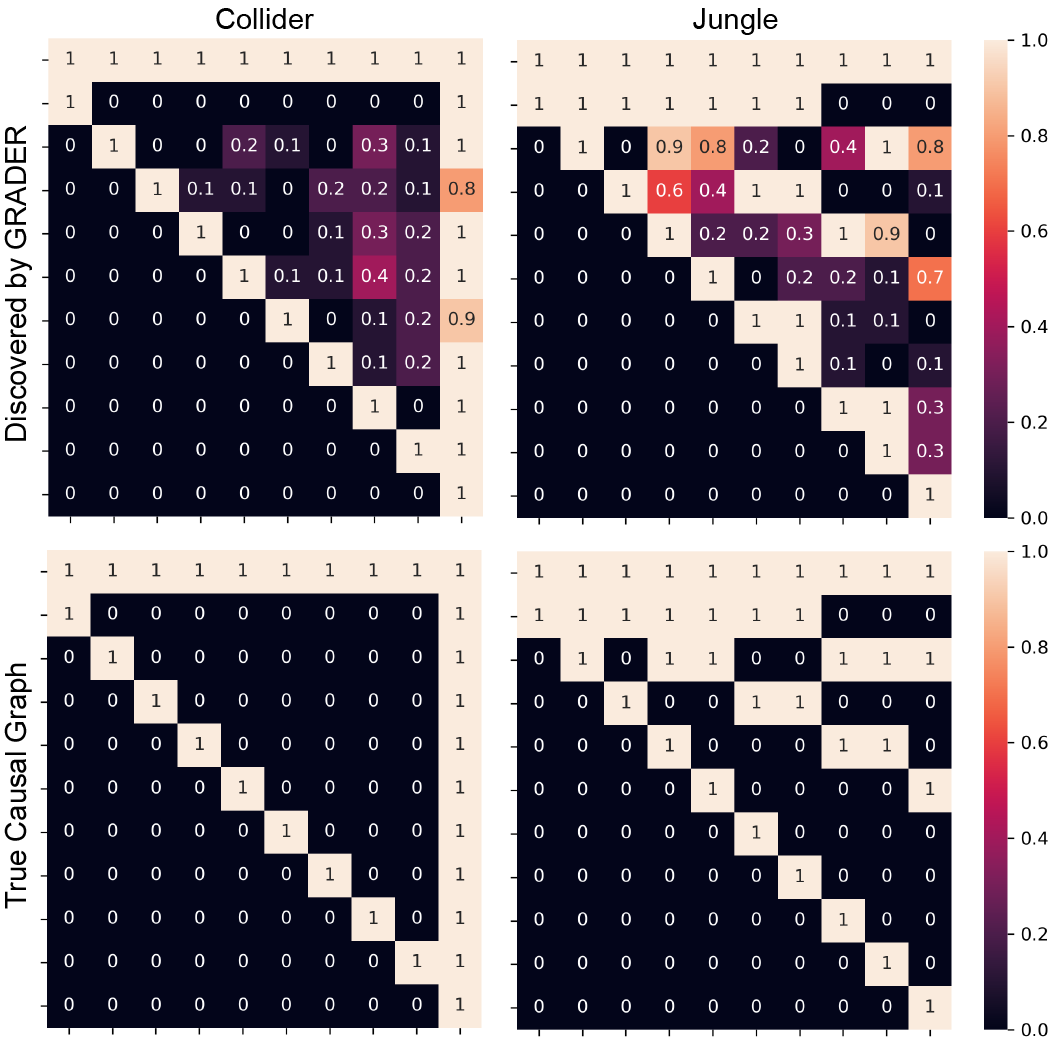}
    \caption{Top: Discovered causal graph from GRADER. Bottom: true causal graph.}
    \label{fig:chemistry_graph}
\vspace{-5mm}
\end{wrapfigure}

\subsection{Further Analysis of Causal Discovery}

Finally, we conduct further analysis of the discovery performance on the \textit{Chemistry} environment~\cite{ke2021systematic}, which is a standard benchmark for evaluating causal discovery methods. In this environment, the colors of nodes are controlled by the causal graph, therefore, finding the true causal graph makes it much easier to achieve the goal that requires matching all target colors.
The agent can discover the graph by doing interventions via interacting with the environment.
We consider four types of causal graphs (\textit{Collider}, \textit{Chain}, \textit{Jungle}, \textit{Full}) with 10 nodes in this experiment. 

The discovery performance is shown in Table~\ref{tab:chemistry_discovery} with five metrics indicating the classification error. We can see that GRADER outperforms the Score method in all 4 types of graphs. 
In Figure~\ref{fig:chemistry_graph}, we show the discovered graphs from GRADER (averaged over 10 seeds) and the true causal graphs of \textit{Collider} and \textit{Jungle} settings. 
We also show that GRADER achieves advantages over other baselines in solving the color-matching downstream task, which can be found with detailed experiment results in Appendix~\ref{app:chemistry_experiment}.

\section{Related Works}

\textbf{RL Generalization}
From the agent’s view, algorithms focus on actively obtaining the structure of the task or environment. 
Some decompose the given task into sub-tasks~\cite{kipf2019compile, xu2019regression, huang2019neural, lu2021learning} and when they encounter an unseen task, they rearrange those sub-tasks to solve new tasks.
Instead of dividing tasks, Symbolic RL learns a program-based policy consisting of domain language~\cite{yang2021program, han2020neuro} or grammar~\cite{landajuela2021discovering, garnelo2016towards}. The generated program is then executed~\cite{zhao2021proto} to interact with the environment, which has the potential to solve unseen tasks by searching the symbolic space.
From the other view, we can generate environments to augment agents' experience for better generalization. 
One straightforward way is data augmentation of image-based observation~\cite{kostrikov2020image, hansen2021stabilizing, srinivas2020curl, hansen2021generalization, lee2019network, raileanu2021automatic}. When extended to other factors of the environment, \textit{Domain Randomization}~\cite{tobin2017domain} and its variants~\cite{mehta2020active, prakash2019structured} are proposed. Considering the interaction between agent and environment, \textit{Curriculum Learning}~\cite{narvekar2020curriculum} also gradually generates difficult tasks to train generalizable agents.

\textbf{Goal-Conditioned RL (GCRL)}
The generalization problem is naturally related to GCRL~\cite{liu2022goal}, which aims to train an agent for multiple tasks.
From the optimization perspective, Universal Value Function~\cite{schaul2015universal}, reward shaping~\cite{trott2019keeping}, and latent dynamic model~\cite{nair2020goal} are widely used tools to solve GCRL problem.
Sub-goal generation~\cite{florensa2018automatic} is another intuitive idea to tackle the long horizon with sparse reward, where the core thing is to make sure the generated sub-goals are solvable.
Finally, \textit{Hindsight Experience Replay (HER)}~\cite{andrychowicz2017hindsight}, belonging to the relabelling category, is a ground-breaking yet straightforward method that treats visited states as ``fake'' goals when the goal and state share the same space. Later on, improved versions of HER~\cite{ren2019exploration, pitis2020maximum, fang2019curriculum} were widely studied. One limitation is that we cannot directly use a visited state as a goal if the goal has pre-conditions.
Similar to our setting, \cite{tang2021hindsight} and \cite{rudner2021outcome} convert the GCRL problem to variational inference by regarding control as inference~\cite{levine2018reinforcement}. \cite{tang2021hindsight} propose an EM framework under the HER setting and \cite{rudner2021outcome} treats the last state as the goal and estimates a shaped reward during training.

\textbf{RL with Causal Reasoning}
Causality is now frequently discussed in the machine learning field to complement the interpretability of neural networks~\cite{peters2017elements}. 
RL algorithms also incorporate causality to improve the reasoning capability~\cite{madumal2020explainable}.
For instance, \cite{nair2019causal} and \cite{volodin2020resolving} explicitly estimate causal structures with the interventional data obtained from the environment. 
These structures can be used to constraint output space~\cite{ding2021causalaf} or adjust the buffer priority~\cite{seitzer2021causal}. Building dynamic models in model-based RL~\cite{wangtask, wang2022causal, zhu2022offline} based on causal graphs is also widely studied recently.
Implicitly, we can abstract the causal structure and formulate it using the \textit{Block MDP}~\cite{zhang2020invariant} setting or training multiple encoders to extract different kinds of representations~\cite{sodhani2022improving}.
Following the idea of invariant risk minimization~\cite{arjovsky2019invariant}, they assume task-relevant features are invariant and shared across all environments, which can be used as the only cause of the reward.

\textbf{Causal Discovery}
Causal discovery~\cite{glymour2019review} is a long-stand topic in economics and sociology, where the traditional methods can be generally categorized into constraint-based and score-based. 
Constraint-based methods~\cite{spirtes2000causation} start from a complete graph and iteratively remove edges with conditional independent test~\cite{pearson1900x, zhang2012kernel} as constraints.
Score-based methods~\cite{chickering2002optimal, hauser2012characterization} use metrics such as \textit{Bayesian Information Criterion}~\cite{neath2012bayesian} as scores and prefer edges that maximize the score given the dataset.
Recently, researchers extend score-based methods with RL~\cite{zhu2019causal} or differentiable discovery~\cite{brouillard2020differentiable, ke2019learning, li2020causal}. The former selects edges with a learned policy, and the latter learns a soft adjacent matrix with observational or interventional data.
Active intervention methods are also explored~\cite{scherrer2021learning} to increase the efficiency of data collection and decrease the cost of conducting intervention~\cite{lindgren2018experimental}.

\section{Conclusion}

This paper proposes a latent variable model that injects a causal graph reasoning process into transition model learning and planning to solve GCRL problems under the generalization setting. We theoretically prove that our iterative optimization process can obtain the true causal graph. 
To evaluate the performance of the proposed method, we designed nine tasks in three environments.
The comprehensive experiment results show that our method has better data efficiency and performance than baselines. 
Our method also provides interpretability by the explicitly discovered causal graph. 

The main limitation of this work is that the explicit estimation of causal structure does not scale well to the number of nodes. Developing efficient gradient-based discovery methods could be a promising direction.
In addition, the factorized state and action space assumption may restrict the usage of this work to semantic representations, which need to be processed with abstraction methods.
We further discuss the potential negative social impact and additional limitations in Appendix~\ref{app:social}.

\paragraph{Acknowledgements.}
We gratefully acknowledge support from the National Science Foundation under grant CAREER CNS-2047454.

\newpage
\bibliographystyle{unsrt}
\bibliography{neurips}

\begin{thebibliography}{10}

\bibitem{xu2021bayesian}
Kai Xu, Akash Srivastava, Dan Gutfreund, Felix Sosa, Tomer Ullman, Josh
  Tenenbaum, and Charles Sutton.
\newblock A bayesian-symbolic approach to reasoning and learning in intuitive
  physics.
\newblock {\em Advances in Neural Information Processing Systems}, 34, 2021.

\bibitem{cranmer2006discovering}
Miles Cranmer, Alvaro Sanchez-Gonzalez, Peter Battaglia, Rui Xu, Kyle Cranmer,
  David Spergel, and Shirley Ho.
\newblock Discovering symbolic models from deep learning with inductive biases
  (2020).
\newblock {\em arXiv preprint arXiv:2006.11287}, 2006.

\bibitem{li2020ngs}
Qing Li, Siyuan Huang, Yining Hong, Yixin Chen, Ying~Nian Wu, and Song-Chun.
  Zhu.
\newblock Closed loop neural-symbolic learning via integrating neural
  perception, grammar parsing, and symbolic reasoning.
\newblock In {\em International Conference on Machine Learning (ICML)}, 2020.

\bibitem{sutton2018reinforcement}
Richard~S Sutton and Andrew~G Barto.
\newblock {\em Reinforcement learning: An introduction}.
\newblock MIT press, 2018.

\bibitem{lu2021learning}
Yuchen Lu, Yikang Shen, Siyuan Zhou, Aaron Courville, Joshua~B Tenenbaum, and
  Chuang Gan.
\newblock Learning task decomposition with ordered memory policy network.
\newblock {\em arXiv preprint arXiv:2103.10972}, 2021.

\bibitem{yang2021program}
Yichen Yang, Jeevana~Priya Inala, Osbert Bastani, Yewen Pu, Armando
  Solar-Lezama, and Martin Rinard.
\newblock Program synthesis guided reinforcement learning for partially
  observed environments.
\newblock {\em Advances in Neural Information Processing Systems}, 34, 2021.

\bibitem{landajuela2021discovering}
Mikel Landajuela, Brenden~K Petersen, Sookyung Kim, Claudio~P Santiago, Ruben
  Glatt, Nathan Mundhenk, Jacob~F Pettit, and Daniel Faissol.
\newblock Discovering symbolic policies with deep reinforcement learning.
\newblock In {\em International Conference on Machine Learning}, pages
  5979--5989. PMLR, 2021.

\bibitem{han2020neuro}
Jiankai Sun1 Hao Sun1~Tian Han and Bolei Zhou.
\newblock Neuro-symbolic program search for autonomous driving decision module
  design.
\newblock In {\em Conference on Robot Learning (CoRL)}, 2020.

\bibitem{zhao2021proto}
Zelin Zhao, Karan Samel, Binghong Chen, et~al.
\newblock Proto: Program-guided transformer for program-guided tasks.
\newblock {\em Advances in Neural Information Processing Systems}, 34, 2021.

\bibitem{sun2019program}
Shao-Hua Sun, Te-Lin Wu, and Joseph~J Lim.
\newblock Program guided agent.
\newblock In {\em International Conference on Learning Representations}, 2019.

\bibitem{gershman2017reinforcement}
Samuel~J Gershman.
\newblock Reinforcement learning and causal models.
\newblock {\em The Oxford handbook of causal reasoning}, 1:295, 2017.

\bibitem{zhang2020invariant}
Amy Zhang, Clare Lyle, Shagun Sodhani, Angelos Filos, Marta Kwiatkowska, Joelle
  Pineau, Yarin Gal, and Doina Precup.
\newblock Invariant causal prediction for block mdps.
\newblock In {\em International Conference on Machine Learning}, pages
  11214--11224. PMLR, 2020.

\bibitem{tomar2021model}
Manan Tomar, Amy Zhang, Roberto Calandra, Matthew~E Taylor, and Joelle Pineau.
\newblock Model-invariant state abstractions for model-based reinforcement
  learning.
\newblock {\em arXiv preprint arXiv:2102.09850}, 2021.

\bibitem{sontakke2021causal}
Sumedh~A Sontakke, Arash Mehrjou, Laurent Itti, and Bernhard Sch{\"o}lkopf.
\newblock Causal curiosity: Rl agents discovering self-supervised experiments
  for causal representation learning.
\newblock In {\em International Conference on Machine Learning}, pages
  9848--9858. PMLR, 2021.

\bibitem{sodhani2022improving}
Shagun Sodhani, Sergey Levine, and Amy Zhang.
\newblock Improving generalization with approximate factored value functions.
\newblock In {\em ICLR2022 Workshop on the Elements of Reasoning: Objects,
  Structure and Causality}, 2022.

\bibitem{bica2021invariant}
Ioana Bica, Daniel Jarrett, and Mihaela van~der Schaar.
\newblock Invariant causal imitation learning for generalizable policies.
\newblock {\em Advances in Neural Information Processing Systems}, 34, 2021.

\bibitem{han2021learning}
Beining Han, Chongyi Zheng, Harris Chan, Keiran Paster, Michael Zhang, and
  Jimmy Ba.
\newblock Learning domain invariant representations in goal-conditioned block
  mdps.
\newblock {\em Advances in Neural Information Processing Systems}, 34, 2021.

\bibitem{wangtask}
Zizhao Wang, Xuesu Xiao, Yuke Zhu, and Peter Stone.
\newblock Task-independent causal state abstraction.
\newblock {\em Workshop on Robot Learning: Self-Supervised and Lifelong
  Learning, NeurIPS}, 2021.

\bibitem{volodin2020resolving}
Sergei Volodin, Nevan Wichers, and Jeremy Nixon.
\newblock Resolving spurious correlations in causal models of environments via
  interventions.
\newblock {\em arXiv preprint arXiv:2002.05217}, 2020.

\bibitem{seitzer2021causal}
Maximilian Seitzer, Bernhard Sch{\"o}lkopf, and Georg Martius.
\newblock Causal influence detection for improving efficiency in reinforcement
  learning.
\newblock {\em Advances in Neural Information Processing Systems}, 34, 2021.

\bibitem{gasse2021causal}
Maxime Gasse, Damien Grasset, Guillaume Gaudron, and Pierre-Yves Oudeyer.
\newblock Causal reinforcement learning using observational and interventional
  data.
\newblock {\em arXiv preprint arXiv:2106.14421}, 2021.

\bibitem{abel2022theory}
David Abel.
\newblock A theory of abstraction in reinforcement learning.
\newblock {\em arXiv preprint arXiv:2203.00397}, 2022.

\bibitem{shanahan2022abstraction}
Murray Shanahan and Melanie Mitchell.
\newblock Abstraction for deep reinforcement learning.
\newblock {\em arXiv preprint arXiv:2202.05839}, 2022.

\bibitem{abel2018state}
David Abel, Dilip Arumugam, Lucas Lehnert, and Michael Littman.
\newblock State abstractions for lifelong reinforcement learning.
\newblock In {\em International Conference on Machine Learning}, pages 10--19.
  PMLR, 2018.

\bibitem{nair2019causal}
Suraj Nair, Yuke Zhu, Silvio Savarese, and Li~Fei-Fei.
\newblock Causal induction from visual observations for goal directed tasks.
\newblock {\em arXiv preprint arXiv:1910.01751}, 2019.

\bibitem{liu2022goal}
Minghuan Liu, Menghui Zhu, and Weinan Zhang.
\newblock Goal-conditioned reinforcement learning: Problems and solutions.
\newblock {\em arXiv preprint arXiv:2201.08299}, 2022.

\bibitem{levine2018reinforcement}
Sergey Levine.
\newblock Reinforcement learning and control as probabilistic inference:
  Tutorial and review.
\newblock {\em arXiv preprint arXiv:1805.00909}, 2018.

\bibitem{boutilier2000stochastic}
Craig Boutilier, Richard Dearden, and Mois{\'e}s Goldszmidt.
\newblock Stochastic dynamic programming with factored representations.
\newblock {\em Artificial intelligence}, 121(1-2):49--107, 2000.

\bibitem{peters2017elements}
Jonas Peters, Dominik Janzing, and Bernhard Sch{\"o}lkopf.
\newblock {\em Elements of causal inference: foundations and learning
  algorithms}.
\newblock The MIT Press, 2017.

\bibitem{spirtes2000causation}
Peter Spirtes, Clark~N Glymour, Richard Scheines, and David Heckerman.
\newblock {\em Causation, prediction, and search}.
\newblock MIT press, 2000.

\bibitem{pitis2020counterfactual}
Silviu Pitis, Elliot Creager, and Animesh Garg.
\newblock Counterfactual data augmentation using locally factored dynamics.
\newblock {\em Advances in Neural Information Processing Systems},
  33:3976--3990, 2020.

\bibitem{abdolmaleki2018maximum}
Abbas Abdolmaleki, Jost~Tobias Springenberg, Yuval Tassa, Remi Munos, Nicolas
  Heess, and Martin Riedmiller.
\newblock Maximum a posteriori policy optimisation.
\newblock {\em arXiv preprint arXiv:1806.06920}, 2018.

\bibitem{marino2019inference}
Joseph Marino and Yisong Yue.
\newblock An inference perspective on model-based reinforcement learning.
\newblock In {\em ICML Workshop on Generative Modeling and Model-Based
  Reasoning for Robotics and AI}, 2019.

\bibitem{andrychowicz2017hindsight}
Marcin Andrychowicz, Filip Wolski, Alex Ray, Jonas Schneider, Rachel Fong,
  Peter Welinder, Bob McGrew, Josh Tobin, OpenAI Pieter~Abbeel, and Wojciech
  Zaremba.
\newblock Hindsight experience replay.
\newblock {\em Advances in neural information processing systems}, 30, 2017.

\bibitem{chung2014empirical}
Junyoung Chung, Caglar Gulcehre, KyungHyun Cho, and Yoshua Bengio.
\newblock Empirical evaluation of gated recurrent neural networks on sequence
  modeling.
\newblock {\em arXiv preprint arXiv:1412.3555}, 2014.

\bibitem{camacho2013model}
Eduardo~F Camacho and Carlos~Bordons Alba.
\newblock {\em Model predictive control}.
\newblock Springer science \& business media, 2013.

\bibitem{richards2005robust}
Arthur~George Richards.
\newblock {\em Robust constrained model predictive control}.
\newblock PhD thesis, Massachusetts Institute of Technology, 2005.

\bibitem{wang2019benchmarking}
Tingwu Wang, Xuchan Bao, Ignasi Clavera, Jerrick Hoang, Yeming Wen, Eric
  Langlois, Shunshi Zhang, Guodong Zhang, Pieter Abbeel, and Jimmy Ba.
\newblock Benchmarking model-based reinforcement learning.
\newblock {\em arXiv preprint arXiv:1907.02057}, 2019.

\bibitem{chickering2002optimal}
David~Maxwell Chickering.
\newblock Optimal structure identification with greedy search.
\newblock {\em Journal of machine learning research}, 3(Nov):507--554, 2002.

\bibitem{chalupka2018fast}
Krzysztof Chalupka, Pietro Perona, and Frederick Eberhardt.
\newblock Fast conditional independence test for vector variables with large
  sample sizes.
\newblock {\em arXiv preprint arXiv:1804.02747}, 2018.

\bibitem{canonne2018testing}
Cl{\'e}ment~L Canonne, Ilias Diakonikolas, Daniel~M Kane, and Alistair Stewart.
\newblock Testing conditional independence of discrete distributions.
\newblock In {\em 2018 Information Theory and Applications Workshop (ITA)},
  pages 1--57. IEEE, 2018.

\bibitem{shah2020hardness}
Rajen~D Shah and Jonas Peters.
\newblock The hardness of conditional independence testing and the generalised
  covariance measure.
\newblock {\em The Annals of Statistics}, 48(3):1514--1538, 2020.

\bibitem{ke2021systematic}
Nan~Rosemary Ke, Aniket Didolkar, Sarthak Mittal, Anirudh Goyal, Guillaume
  Lajoie, Stefan Bauer, Danilo Rezende, Yoshua Bengio, Michael Mozer, and
  Christopher Pal.
\newblock Systematic evaluation of causal discovery in visual model based
  reinforcement learning.
\newblock {\em arXiv preprint arXiv:2107.00848}, 2021.

\bibitem{ahmed2020causalworld}
Ossama Ahmed, Frederik Tr{\"a}uble, Anirudh Goyal, Alexander Neitz, Yoshua
  Bengio, Bernhard Sch{\"o}lkopf, Manuel W{\"u}thrich, and Stefan Bauer.
\newblock Causalworld: A robotic manipulation benchmark for causal structure
  and transfer learning.
\newblock {\em arXiv preprint arXiv:2010.04296}, 2020.

\bibitem{gym_minigrid}
Maxime Chevalier-Boisvert, Lucas Willems, and Suman Pal.
\newblock Minimalistic gridworld environment for openai gym.
\newblock \url{https://github.com/maximecb/gym-minigrid}, 2018.

\bibitem{tavares2021language}
Zenna Tavares, James Koppel, Xin Zhang, Ria Das, and Armando Solar-Lezama.
\newblock A language for counterfactual generative models.
\newblock In {\em International Conference on Machine Learning}, pages
  10173--10182. PMLR, 2021.

\bibitem{highway-env}
Edouard Leurent.
\newblock An environment for autonomous driving decision-making.
\newblock \url{https://github.com/eleurent/highway-env}, 2018.

\bibitem{haarnoja2018soft}
Tuomas Haarnoja, Aurick Zhou, Pieter Abbeel, and Sergey Levine.
\newblock Soft actor-critic: Off-policy maximum entropy deep reinforcement
  learning with a stochastic actor.
\newblock In {\em International conference on machine learning}, pages
  1861--1870. PMLR, 2018.

\bibitem{ross2011no}
St{\'e}phane Ross, Geoffrey~J Gordon, and J~Andrew Bagnell.
\newblock No-regret reductions for imitation learning and structured
  prediction.
\newblock In {\em In AISTATS}. Citeseer, 2011.

\bibitem{chua2018deep}
Kurtland Chua, Roberto Calandra, Rowan McAllister, and Sergey Levine.
\newblock Deep reinforcement learning in a handful of trials using
  probabilistic dynamics models.
\newblock {\em Advances in neural information processing systems}, 31, 2018.

\bibitem{schlichtkrull2018modeling}
Michael Schlichtkrull, Thomas~N Kipf, Peter Bloem, Rianne van~den Berg, Ivan
  Titov, and Max Welling.
\newblock Modeling relational data with graph convolutional networks.
\newblock In {\em European semantic web conference}, pages 593--607. Springer,
  2018.

\bibitem{zhu2022offline}
Zheng-Mao Zhu, Xiong-Hui Chen, Hong-Long Tian, Kun Zhang, and Yang Yu.
\newblock Offline reinforcement learning with causal structured world models.
\newblock {\em arXiv preprint arXiv:2206.01474}, 2022.

\bibitem{tsamardinos2006max}
Ioannis Tsamardinos, Laura~E Brown, and Constantin~F Aliferis.
\newblock The max-min hill-climbing bayesian network structure learning
  algorithm.
\newblock {\em Machine learning}, 65(1):31--78, 2006.

\bibitem{brouillard2020differentiable}
Philippe Brouillard, S{\'e}bastien Lachapelle, Alexandre Lacoste, Simon
  Lacoste-Julien, and Alexandre Drouin.
\newblock Differentiable causal discovery from interventional data.
\newblock {\em arXiv preprint arXiv:2007.01754}, 2020.

\bibitem{kipf2019compile}
Thomas Kipf, Yujia Li, Hanjun Dai, Vinicius Zambaldi, Alvaro Sanchez-Gonzalez,
  Edward Grefenstette, Pushmeet Kohli, and Peter Battaglia.
\newblock Compile: Compositional imitation learning and execution.
\newblock In {\em International Conference on Machine Learning}, pages
  3418--3428. PMLR, 2019.

\bibitem{xu2019regression}
Danfei Xu, Roberto Mart{\'\i}n-Mart{\'\i}n, De-An Huang, Yuke Zhu, Silvio
  Savarese, and Li~F Fei-Fei.
\newblock Regression planning networks.
\newblock {\em Advances in Neural Information Processing Systems}, 32, 2019.

\bibitem{huang2019neural}
De-An Huang, Suraj Nair, Danfei Xu, Yuke Zhu, Animesh Garg, Li~Fei-Fei, Silvio
  Savarese, and Juan~Carlos Niebles.
\newblock Neural task graphs: Generalizing to unseen tasks from a single video
  demonstration.
\newblock In {\em Proceedings of the IEEE/CVF Conference on Computer Vision and
  Pattern Recognition}, pages 8565--8574, 2019.

\bibitem{garnelo2016towards}
Marta Garnelo, Kai Arulkumaran, and Murray Shanahan.
\newblock Towards deep symbolic reinforcement learning.
\newblock {\em arXiv preprint arXiv:1609.05518}, 2016.

\bibitem{kostrikov2020image}
Ilya Kostrikov, Denis Yarats, and Rob Fergus.
\newblock Image augmentation is all you need: Regularizing deep reinforcement
  learning from pixels.
\newblock {\em arXiv preprint arXiv:2004.13649}, 2020.

\bibitem{hansen2021stabilizing}
Nicklas Hansen, Hao Su, and Xiaolong Wang.
\newblock Stabilizing deep q-learning with convnets and vision transformers
  under data augmentation.
\newblock {\em Advances in Neural Information Processing Systems}, 34, 2021.

\bibitem{srinivas2020curl}
Aravind Srinivas, Michael Laskin, and Pieter Abbeel.
\newblock Curl: Contrastive unsupervised representations for reinforcement
  learning.
\newblock {\em arXiv preprint arXiv:2004.04136}, 2020.

\bibitem{hansen2021generalization}
Nicklas Hansen and Xiaolong Wang.
\newblock Generalization in reinforcement learning by soft data augmentation.
\newblock In {\em 2021 IEEE International Conference on Robotics and Automation
  (ICRA)}, pages 13611--13617. IEEE, 2021.

\bibitem{lee2019network}
Kimin Lee, Kibok Lee, Jinwoo Shin, and Honglak Lee.
\newblock Network randomization: A simple technique for generalization in deep
  reinforcement learning.
\newblock {\em arXiv preprint arXiv:1910.05396}, 2019.

\bibitem{raileanu2021automatic}
Roberta Raileanu, Maxwell Goldstein, Denis Yarats, Ilya Kostrikov, and Rob
  Fergus.
\newblock Automatic data augmentation for generalization in reinforcement
  learning.
\newblock {\em Advances in Neural Information Processing Systems}, 34, 2021.

\bibitem{tobin2017domain}
Josh Tobin, Rachel Fong, Alex Ray, Jonas Schneider, Wojciech Zaremba, and
  Pieter Abbeel.
\newblock Domain randomization for transferring deep neural networks from
  simulation to the real world.
\newblock In {\em 2017 IEEE/RSJ international conference on intelligent robots
  and systems (IROS)}, pages 23--30. IEEE, 2017.

\bibitem{mehta2020active}
Bhairav Mehta, Manfred Diaz, Florian Golemo, Christopher~J Pal, and Liam Paull.
\newblock Active domain randomization.
\newblock In {\em Conference on Robot Learning}, pages 1162--1176. PMLR, 2020.

\bibitem{prakash2019structured}
Aayush Prakash, Shaad Boochoon, Mark Brophy, David Acuna, Eric Cameracci,
  Gavriel State, Omer Shapira, and Stan Birchfield.
\newblock Structured domain randomization: Bridging the reality gap by
  context-aware synthetic data.
\newblock In {\em 2019 International Conference on Robotics and Automation
  (ICRA)}, pages 7249--7255. IEEE, 2019.

\bibitem{narvekar2020curriculum}
Sanmit Narvekar, Bei Peng, Matteo Leonetti, Jivko Sinapov, Matthew~E Taylor,
  and Peter Stone.
\newblock Curriculum learning for reinforcement learning domains: A framework
  and survey.
\newblock {\em arXiv preprint arXiv:2003.04960}, 2020.

\bibitem{schaul2015universal}
Tom Schaul, Daniel Horgan, Karol Gregor, and David Silver.
\newblock Universal value function approximators.
\newblock In {\em International conference on machine learning}, pages
  1312--1320. PMLR, 2015.

\bibitem{trott2019keeping}
Alexander Trott, Stephan Zheng, Caiming Xiong, and Richard Socher.
\newblock Keeping your distance: Solving sparse reward tasks using
  self-balancing shaped rewards.
\newblock {\em Advances in Neural Information Processing Systems}, 32, 2019.

\bibitem{nair2020goal}
Suraj Nair, Silvio Savarese, and Chelsea Finn.
\newblock Goal-aware prediction: Learning to model what matters.
\newblock In {\em International Conference on Machine Learning}, pages
  7207--7219. PMLR, 2020.

\bibitem{florensa2018automatic}
Carlos Florensa, David Held, Xinyang Geng, and Pieter Abbeel.
\newblock Automatic goal generation for reinforcement learning agents.
\newblock In {\em International conference on machine learning}, pages
  1515--1528. PMLR, 2018.

\bibitem{ren2019exploration}
Zhizhou Ren, Kefan Dong, Yuan Zhou, Qiang Liu, and Jian Peng.
\newblock Exploration via hindsight goal generation.
\newblock {\em Advances in Neural Information Processing Systems}, 32, 2019.

\bibitem{pitis2020maximum}
Silviu Pitis, Harris Chan, Stephen Zhao, Bradly Stadie, and Jimmy Ba.
\newblock Maximum entropy gain exploration for long horizon multi-goal
  reinforcement learning.
\newblock In {\em International Conference on Machine Learning}, pages
  7750--7761. PMLR, 2020.

\bibitem{fang2019curriculum}
Meng Fang, Tianyi Zhou, Yali Du, Lei Han, and Zhengyou Zhang.
\newblock Curriculum-guided hindsight experience replay.
\newblock {\em Advances in neural information processing systems}, 32, 2019.

\bibitem{tang2021hindsight}
Yunhao Tang and Alp Kucukelbir.
\newblock Hindsight expectation maximization for goal-conditioned reinforcement
  learning.
\newblock In {\em International Conference on Artificial Intelligence and
  Statistics}, pages 2863--2871. PMLR, 2021.

\bibitem{rudner2021outcome}
Tim~GJ Rudner, Vitchyr Pong, Rowan McAllister, Yarin Gal, and Sergey Levine.
\newblock Outcome-driven reinforcement learning via variational inference.
\newblock {\em Advances in Neural Information Processing Systems}, 34, 2021.

\bibitem{madumal2020explainable}
Prashan Madumal, Tim Miller, Liz Sonenberg, and Frank Vetere.
\newblock Explainable reinforcement learning through a causal lens.
\newblock In {\em Proceedings of the AAAI conference on artificial
  intelligence}, volume~34, pages 2493--2500, 2020.

\bibitem{ding2021causalaf}
Wenhao Ding, Haohong Lin, Bo~Li, and Ding Zhao.
\newblock Causalaf: Causal autoregressive flow for goal-directed
  safety-critical scenes generation.
\newblock {\em arXiv preprint arXiv:2110.13939}, 2021.

\bibitem{wang2022causal}
Zizhao Wang, Xuesu Xiao, Zifan Xu, Yuke Zhu, and Peter Stone.
\newblock Causal dynamics learning for task-independent state abstraction.
\newblock {\em arXiv preprint arXiv:2206.13452}, 2022.

\bibitem{arjovsky2019invariant}
Martin Arjovsky, L{\'e}on Bottou, Ishaan Gulrajani, and David Lopez-Paz.
\newblock Invariant risk minimization.
\newblock {\em arXiv preprint arXiv:1907.02893}, 2019.

\bibitem{glymour2019review}
Clark Glymour, Kun Zhang, and Peter Spirtes.
\newblock Review of causal discovery methods based on graphical models.
\newblock {\em Frontiers in genetics}, 10:524, 2019.

\bibitem{pearson1900x}
Karl Pearson.
\newblock X. on the criterion that a given system of deviations from the
  probable in the case of a correlated system of variables is such that it can
  be reasonably supposed to have arisen from random sampling.
\newblock {\em The London, Edinburgh, and Dublin Philosophical Magazine and
  Journal of Science}, 50(302):157--175, 1900.

\bibitem{zhang2012kernel}
Kun Zhang, Jonas Peters, Dominik Janzing, and Bernhard Sch{\"o}lkopf.
\newblock Kernel-based conditional independence test and application in causal
  discovery.
\newblock {\em arXiv preprint arXiv:1202.3775}, 2012.

\bibitem{hauser2012characterization}
Alain Hauser and Peter B{\"u}hlmann.
\newblock Characterization and greedy learning of interventional markov
  equivalence classes of directed acyclic graphs.
\newblock {\em The Journal of Machine Learning Research}, 13(1):2409--2464,
  2012.

\bibitem{neath2012bayesian}
Andrew~A Neath and Joseph~E Cavanaugh.
\newblock The bayesian information criterion: background, derivation, and
  applications.
\newblock {\em Wiley Interdisciplinary Reviews: Computational Statistics},
  4(2):199--203, 2012.

\bibitem{zhu2019causal}
Shengyu Zhu, Ignavier Ng, and Zhitang Chen.
\newblock Causal discovery with reinforcement learning.
\newblock {\em arXiv preprint arXiv:1906.04477}, 2019.

\bibitem{ke2019learning}
Nan~Rosemary Ke, Olexa Bilaniuk, Anirudh Goyal, Stefan Bauer, Hugo Larochelle,
  Bernhard Sch{\"o}lkopf, Michael~C Mozer, Chris Pal, and Yoshua Bengio.
\newblock Learning neural causal models from unknown interventions.
\newblock {\em arXiv preprint arXiv:1910.01075}, 2019.

\bibitem{li2020causal}
Yunzhu Li, Antonio Torralba, Anima Anandkumar, Dieter Fox, and Animesh Garg.
\newblock Causal discovery in physical systems from videos.
\newblock {\em Advances in Neural Information Processing Systems},
  33:9180--9192, 2020.

\bibitem{scherrer2021learning}
Nino Scherrer, Olexa Bilaniuk, Yashas Annadani, Anirudh Goyal, Patrick Schwab,
  Bernhard Sch{\"o}lkopf, Michael~C Mozer, Yoshua Bengio, Stefan Bauer, and
  Nan~Rosemary Ke.
\newblock Learning neural causal models with active interventions.
\newblock {\em arXiv preprint arXiv:2109.02429}, 2021.

\bibitem{lindgren2018experimental}
Erik Lindgren, Murat Kocaoglu, Alexandros~G Dimakis, and Sriram Vishwanath.
\newblock Experimental design for cost-aware learning of causal graphs.
\newblock {\em Advances in Neural Information Processing Systems}, 31, 2018.

\bibitem{Gilboa2019DynamicalIA}
Dar Gilboa, B.~Chang, Minmin Chen, Greg Yang, Samuel~S. Schoenholz, Ed~H. Chi,
  and Jeffrey Pennington.
\newblock Dynamical isometry and a mean field theory of lstms and grus.
\newblock {\em ArXiv}, abs/1901.08987, 2019.

\bibitem{Yang2018CharacterizingAL}
Karren~D. Yang, Abigail Katoff, and Caroline Uhler.
\newblock Characterizing and learning equivalence classes of causal dags under
  interventions.
\newblock In {\em ICML}, 2018.

\bibitem{pmlr-v119-addanki20a}
Raghavendra Addanki, Shiva Kasiviswanathan, Andrew Mcgregor, and Cameron Musco.
\newblock Efficient intervention design for causal discovery with latents.
\newblock In Hal~Daumé III and Aarti Singh, editors, {\em Proceedings of the
  37th International Conference on Machine Learning}, volume 119 of {\em
  Proceedings of Machine Learning Research}, pages 63--73. PMLR, 13--18 Jul
  2020.

\bibitem{ding2022survey}
Wenhao Ding, Chejian Xu, Haohong Lin, Bo~Li, and Ding Zhao.
\newblock A survey on safety-critical scenario generation from methodological
  perspective.
\newblock {\em arXiv preprint arXiv:2202.02215}, 2022.

\end{thebibliography}
\medskip


\newpage
\section*{Checklist}
\begin{enumerate}
\item For all authors...
\begin{enumerate}
  \item Do the main claims made in the abstract and introduction accurately reflect the paper's contributions and scope?
    \answerYes{}
  \item Did you describe the limitations of your work?
    \answerYes{In the conclusion section.}
  \item Did you discuss any potential negative societal impacts of your work?
    \answerYes{In the Appendix.}
  \item Have you read the ethics review guidelines and ensured that your paper conforms to them?
    \answerYes{}
\end{enumerate}

\item If you are including theoretical results...
\begin{enumerate}
  \item Did you state the full set of assumptions of all theoretical results?
    \answerYes{}
        \item Did you include complete proofs of all theoretical results?
    \answerYes{In the Appendix.}
\end{enumerate}

\item If you ran experiments...
\begin{enumerate}
  \item Did you include the code, data, and instructions needed to reproduce the main experimental results (either in the supplemental material or as a URL)?
    \answerYes{In Supplementary.}
  \item Did you specify all the training details (e.g., data splits, hyperparameters, how they were chosen)?
    \answerYes{In the Appendix.}
        \item Did you report error bars (e.g., with respect to the random seed after running experiments multiple times)?
    \answerYes{In the Appendix.}
        \item Did you include the total amount of compute and the type of resources used (e.g., type of GPUs, internal cluster, or cloud provider)?
    \answerYes{In the Appendix.}
\end{enumerate}

\item If you are using existing assets (e.g., code, data, models) or curating/releasing new assets...
\begin{enumerate}
  \item If your work uses existing assets, did you cite the creators?
    \answerYes{In the Supplementary.}
  \item Did you mention the license of the assets?
    \answerYes{In the Supplementary.}
  \item Did you include any new assets either in the supplemental material or as a URL?
    \answerYes{In the Supplementary.}
  \item Did you discuss whether and how consent was obtained from people whose data you're using/curating?
    \answerYes{In the Supplementary}
  \item Did you discuss whether the data you are using/curating contains personally identifiable information or offensive content?
    \answerYes{In the Supplementary.}
\end{enumerate}

\item If you used crowdsourcing or conducted research with human subjects...
\begin{enumerate}
  \item Did you include the full text of instructions given to participants and screenshots, if applicable?
    \answerNA{}
  \item Did you describe any potential participant risks, with links to Institutional Review Board (IRB) approvals, if applicable?
    \answerNA{}
  \item Did you include the estimated hourly wage paid to participants and the total amount spent on participant compensation?
    \answerNA{}
\end{enumerate}
\end{enumerate}

\clearpage
\appendix

\addcontentsline{toc}{section}{Appendix} 
\part{Appendix} 
\parttoc 

\section{Theoretical Proofs}
\label{proof}

In Appendix~\ref{proof}, we first show the derivation of the latent variable models~\ref{app:lvm}, then provide some analytical results in the iterative optimization of model learning~\ref{app:transition}, planning~\ref{app:planning}, and causal discovery~\ref{app:causal_discovery}. Finally, we give the proof of the theorem of overall performance guarantee~\ref{app:convergence} given some common assumptions.

To compactly write down our formulas, we slightly abuse the notations by representing $s^t, a^t$ as the joint states and actions at timestep $t$, while using $s_i, a_i$ to denote the $i$-th dimension of factorized states or actions. Without the loss of generality, we implement our reward in a deterministic way, which only involves $r(s, g)$ in its notation, we will slightly generalize to some state-action reward function for our analysis as well.

\subsection{Formulation of Latent Variable Models}
\label{app:lvm}
\subsubsection{Derivation of Equation (\ref{elbo})}
\label{app:derivation_elbo}

The ELBO of the likelihood of the trajectory is obtained by
\begin{equation}
\begin{split}
    \log p(\tau|s^*) =& \log \int p(\tau|\gG, s^*) p(\gG| s^*) d\gG \\
    =& \log \int q(\gG|\tau)\frac{p(\tau|\gG, s^*) p(\gG|s^*)}{q(\gG|\tau)} d\gG \\
    \geq & \int q(\gG|\tau)\log \frac{p(\tau|\gG, s^*) p(\gG|s^*)}{q(\gG|\tau)} d\gG \\
    =& \int q(\gG|\tau)\left(\log p(\tau|\gG, s^*) + \log \frac{p(\gG | s^*)}{q(\gG|\tau)} \right) d\gG \\
    =& \int q(\gG|\tau)\log p(\tau|\gG, s^*) d\gG + \int q(\gG|\tau)\log\frac{p(\gG | s^*)}{q(\gG|\tau)}  d\gG \\
    =&\ \mathbb{E}_{q(\gG|\tau)}[\log p(\tau|\gG, s^*)] - \mathbb{D}_{\text{KL}}[q(\gG|\tau) || p(\gG)] \\
\end{split}
\end{equation}

where the third line is obtained by Jensen's inequality and the last line is because the prior of the causal graph $\gG$ is independent of the achieved goal $s^*$.

\subsubsection{Derivation of Equation (\ref{transition})}
\label{app:derivation_transition}

According to the decomposition of state-action trajectory
\begin{equation}
    p(\tau) = p(s^{0}) \sum_{t=0}^{T-1} p(s^{t+1}|s^{t}, a^{t}) p(a^{t}|s^{t})
\end{equation}
we can get the following
\begin{equation}
\begin{split}
    \log p(\tau|\gG, s^*) = & \log(s^0, a^0, s^1, a^1, \cdots,a^{T-1}, s^T | \gG, s^*) \\
    =& \log p(s^0|\gG, s^*) + \sum_{t=0}^{T-1} \log p(s^{t+1}|s^{t}, a^{t}, \gG, s^*) + \sum_{t=0}^{T-1} \log p(a^{t}|s^{t}, \gG, s^*) \\ 
    =& \log p(s^0) + \sum_{t=0}^{T-1} \log p(s^{t+1}|s^{t}, a^{t}, \gG) + \sum_{t=0}^{T-1} \log p(a^{t}|s^{t}, \gG, s^*)\\ 
\end{split}
\end{equation}

\subsection{Transition Model Learning}
\label{app:transition}
The optimization in the model learning step can be described below:
\begin{equation}
    \begin{aligned}
        \argmax_{\theta}  [\sum_t \log p_\theta(s_{t+1} | s_t, a_t, \gG)]
    \end{aligned}
\end{equation}
where $\tau=[s_1, a_1, \cdots, s_T]$ is the trajectory in data buffer, and $\gG$ is the given causal graph.

Here below, we show some necessary definitions and propositions to prove the Lemma~\ref{monotonicity}. 
\label{proof_monotonicity}

\begin{definition}[Structural Hamming Distance (SHD)]
\label{shd}
For any two DAGs $\gG, \gH$ with identical vertices set $V$, we define the following function SHD: $\gG\times \gH\to \mathbb{R}$, 
\begin{equation}
    \text{SHD}(\gG, \gH) = \#\{(i, j)\in V^2 \mid \gG \text{ and } \gH \text{ have different edges } e_{ij}\}
\end{equation}
\end{definition}

\begin{definition}[Respect the graph]
For any given transition model with specific causal graph $\gG$, the transition model respects the graph if the distribution $p({s}_{t+1}| {a}_{t}, {s}_{t}, \gG)$ can be factorized as:
\begin{equation}
    p({s'}| {s},{a},  \gG) = \prod_{i\in [M]} p(s_i'| \textbf{PA}(s_i'), \gG)
\end{equation}
where $M$ is the total number of factorized states, $\textbf{PA}(\cdot)$ represents the parents in the causal graph.
\end{definition}
\begin{proposition}[GRU model respects the graph] 
\label{app:gru_repsect}
As the parameterized transition model $p_\theta(s'| s, a, \gG)$ reaches the steady state, it respects the graph.
\end{proposition}

\begin{proof}[Proof of Proposition~\ref{app:gru_repsect}]
The GRU modules with parameter $\theta=[W, U]$ can be rewritten as a message passing process, where $AGG_{\cdot}(\cdot)$ is the iterative aggregation function.
\begin{equation}
\begin{aligned}
    \text{Node Encoder} &: h_j^{(0)} = f_{\text{encoder}}(x_j)\\
    \text{Aggregation} &: h_j^{(\ell)} = AGG_{i\in \mathcal{N}(j)} (f_\theta(x_j^{(\ell-1)}, h_i^{(\ell-1)}))\\
    \text{Node Decoder} &: \bm{x}_i^{(\ell)} = f_{\text{decoder}} (h_j^{(\ell-1)})
\end{aligned}
\end{equation}
As an iterated process of message passing, where the input causal graph controls the information flow between different entities, this GRU model can be rewritten as a fixed point iteration \cite{Gilboa2019DynamicalIA}:
\begin{equation}
    \bm{x}_i^{(\ell)} = F_\theta (\textbf{PA}(\bm{x}_i)^{(\ell-1)}, \bm{x}_i^{(\ell-1)})
\end{equation}
With proper initialization and some sufficient conditions provided by \cite{Gilboa2019DynamicalIA}, $F$ has a unique equilibrium point, where
\begin{equation}
    \bm{x}_i^{\infty} = F_\theta (\textbf{PA}(\bm{x}_i)^{\infty}, \bm{x}_i^{\infty})
\end{equation}
In our bipartite graph, when GRU reaches the equilibrium point, we can get a structural causal model:
\begin{equation} 
\label{eq::scm}
    s_i' = F_\theta(\textbf{PA}(s_i'), s_i), \quad \text{where } s_i'\in \gS', \textbf{PA}(s_i')\in \gS\cup \gA
\end{equation}
Based on the SCM derivation \ref{eq::scm}, we can then factorize the transition model as:
\begin{equation}\label{eq::factor}
    p_\theta(s'| s, a, \gG) = \prod_{i\in [M]} p_\theta(s_i'| \textbf{PA}(s_i'), \gG)
\end{equation}
We denote the ground truth causal graph as $G^*=(V, E^*)$, and $\textbf{PA}^*(s_i')$ as the true parents of $s_i'$ in $G^*$.
\end{proof}

\begin{definition}[Causal optimality at equlibrium point]
For any $G'\neq G^*$ with at least one pair of flawed parental relationship $\textbf{PA}'(s_i')\neq \textbf{PA}^*(s_i')$, the following inequality holds:
\begin{equation}
    p_\theta(s_i'| \textbf{PA}'(s_i'), \gG) \leq p_\theta(s_i'| \textbf{PA}^*(s_i'), \gG)
\end{equation}
\end{definition}

\begin{lemma}[Local monotonicity] \label{local_monotonic}
    Given one state variable $s_i$ and its any parental relationship $\textbf{PA}^1(s_i), \textbf{PA}^2(s_i)$, if $\#(\textbf{PA}^1(s_i)\cup \textbf{PA}^*(s_i))\geq \#(\textbf{PA}^2(s_i)\cup \textbf{PA}^*(s_i))$, then at steady state, the SCM derived in \ref{eq::scm} will miss part of the message provided from the true parents, therefore $p_\theta(s_i'| \textbf{PA}^1(s_i'), \gG_1) \geq p_\theta(s_i'| \textbf{PA}^2(s_i'), \gG_2)$
\end{lemma}

\begin{proof}[Proof of Lemma~\ref{monotonicity}]
Based on the factorization defined in \ref{eq::factor}, we denote the parental relationship in $\gG_1$ as $\textbf{PA}^1(\cdot)$,
\begin{equation}\label{ineq1}
    p_\theta(s'| a, s, \gG^*) = \prod_{i\in [M]} p_\theta(s_i'| \textbf{PA}^*(s_i'), \gG) \geq \prod_{i\in [M]} p_\theta(s_i'| \textbf{PA}^1(s_i'), \gG) = p_\theta(s'| a, s, \gG_1) 
\end{equation}
For $\gG_1$ and $\gG_2$, suppose the only different edges $e$ has a target node $s_j'$, with $\text{SHD}(\gG_1, \gG^*) < \text{SHD}(\gG_2, \gG^*)$, based on Lemma~\ref{local_monotonic}:
\begin{equation} \label{ineq2}
    \begin{aligned}
    p_\theta(s' | a,  s, \gG_1) & = p_\theta (s_j'|\textbf{PA}^1(s_j'), \gG_1) \prod_{i\in [M]\backslash j}  p_\theta(s_i'| \textbf{PA}^1(s_i'), \gG_1) \\
    & \geq p_\theta (s_j'|\textbf{PA}^2(s_j'), \gG_2) \prod_{i\in [M]\backslash j}  p_\theta(s_i'| \textbf{PA}^1(s_i'), \gG_2) \\
    &  =p_\theta (s_j'|\textbf{PA}^2(s_j'), \gG_2) \prod_{i\in [M]\backslash j}  p_\theta(s_i'| \textbf{PA}^2(s_i'), \gG_2) \\
    &= p_\theta(s' | a,  s, \gG_2).
    \end{aligned}
\end{equation}
Based on the inequality derived in \ref{ineq1} and \ref{ineq2}, 
\begin{equation}
    \log p_\theta(s' |s, a,   \gG^*) \geq \log p_\theta(s' |  s, a, \gG_1) \geq \log p_\theta(s' |  s, a, \gG_2).
\end{equation}
the monotonicity of likelihood in Lemma~\ref{monotonicity} is proved.
\end{proof}

\subsection{Derivation of Planner Module}
\label{app:planning}

The optimization in the planning part is: 
\begin{equation}
    \max_{\pi} \sum_{t=0}^{T-1} \log \pi_\theta(a^{t}|s^t, \gG, s^*) = \max_{[a_0, \cdots, a^{T-1}]} \sum_{t=0}^{T-1} \log \hat{Q}(s^t, a^t)
\end{equation}

Ideally, given the access to real dynamics $p(s'|s, a)$ and goal distribution $p(g)$.
We first define the expected goal-conditioned state-action reward $r(s, a, g) = \mathbb{E}_{s'\sim p(\cdot|s, a)} r(s', g)$, and the expected state-action reward $r(s, a) = \mathbb{E}_{g\sim p_g(\cdot)} r(s, a, g)$.
In practice, due to the inaccuracy of transition model, we can only query the following reward estimation at certain state-action pair: $r(s, a, g) = \mathbb{E}_{s'\sim p_\theta(\cdot|s, a, \gG)} r(s', g), r(s, a) = \mathbb{E}_{g\sim p_g(\cdot)} r(s, a, g)$.

Then we consider the distribution of the goal $g\sim p_g(\cdot)$, which is supported on the state space $\gS$. Based on Algorithm \ref{algorithm1}, our interventional data is collected by the MPC that maximizes the expected discounted cumulative reward from learned dynamics. Thus, we could denote the interventional distribution of state (depending on the current policy $\pi$) in the data buffer as $p_{\gI_\pi^s}$, $s\sim p_{\gI_\pi^s}(\cdot)$ which is also supported on the state space $\gS$.

\begin{proof}[Proof of Lemma~\ref{TV_distance}]
Assume our planning algorithm has an infinite planning horizon, with the optimal transition dynamics and optimal policy, the action-value function $Q^*$ can be expressed as:
\begin{equation}
    \begin{aligned}
        Q^*(s^t, a^t)       &\overset{\text{def}}{=}  \mathbb{E}_{s\sim {p_{\gI_{\pi^*}^s}}(\cdot), a\sim \pi^*(s)} \left[\sum_{t'=t}^{\infty} \gamma^{t'-t} r(s^{t'}, a^{t'}) \mid s^t, a^t \right]\\ 
    \end{aligned}
\end{equation}

The estimation of action value function $\hat{Q}(s^t, a^t) = Q_{\theta, \gG}^{\hat{\pi}}(s^t, a^t)$ can be written as:
\begin{equation}
    \begin{aligned}
        \hat{Q} (s^t, a^t)       &\overset{\text{def}}{=}  \mathbb{E}_{s\sim {p_{\gI_{\hat{\pi}}^s}}(\cdot), a\sim \hat{\pi}(s)} \left[\sum_{t'=t}^{\infty} \gamma^{t'-t} r(s^{t'}, a^{t'}) \mid s^t, a^t \right]\\
        & = \mathbb{E}_{a\sim \hat{\pi}(s)} \left[\sum_{t'=t}^{\infty} \gamma^{t'-t} \mathbb{E}_{g\sim p_g(\cdot), s\sim p_{\gI_{\hat{\pi}}^s}(\cdot)} (1 - \mathds{1}(s'=g)) \mid s^t, a^t \right]\\
    \end{aligned}
\end{equation}

The policy by MPC in algorithm~\ref{algorithm1} can be deducted by: $\hat{\pi}(s^t) = \argmax_{a^t\in \gA} \hat{Q}(s^t, a^t)$, let $s^0=s$, and we could derive value function under the MPC policy as follows: 
\begin{equation}
    \begin{aligned}
    V(s) & = \sum_{t=0}^{\infty} \gamma^{t}  \mathbb{E}_{p_g} \mathbb{E}_{p_{\gI_\pi^s}} \mathds{1}(s=g) = \sum_{t=0}^{\infty} \gamma^{t}  \mathbb{E}_{p_g} \mathbb{E}_{p_{\gI_\pi^s}} [1 - \mathds{1}(s\ne g)] \\
    & = \sum_{t=0}^{\infty} \gamma^t - \sum_{t=0}^\infty \gamma^t \mathbb{E}_{p_g} \mathbb{E}_{p_{\gI_\pi^s}} \mathds{1}(s\ne g) \\
    & \leq \frac{1 - \mathbb{D}_{TV}(p_{\gI_{\hat{\pi}}^s}, p_g)}{1-\gamma}
    \end{aligned}
\label{eq:distance_bound}
\end{equation}
where $ \mathbb{D}_{\text{TV}}(p_{\gI_\pi^s}, p_g)$ is the total variation distance between the marginal state distribution $p_{\gI_\pi^s}$ in the data buffer, as well as the goal distribution $p_g$, both of which share the same support.
\end{proof}

In addition, we could define a more general form of goal-conditioned reward as based on the distance: $r(s, g) = 1- d(s', g)$. Where $\mathbb{D}$ is a (normalized) distance measure between two vectors in the state space, s.t. $\forall s, g\in \gS$, $0\leq d(s, g) \leq 1$. 
For instance, if we pick $d(s,g) = \mathds{1}(s\ne g)$, the derived reward under this distance measure will go back to the reward function defined in section~\ref{FactorMDP}.
By defining a (normalized) $\ell_p$ distance between s' and g, $d(s, g) = \frac{\|s-g\|_p}{\max_{s_1, s_2\in \gS} \|s_1-s_2\|_p}$, we can also shape a continuous form of goal-conditioned step reward $r(s, g)$ between 0 and 1. Notice that all the Euclidean-based distances are all valid metrics with symmetry, non-negativity, the identity of indiscernibles, and the triangle inequality. With such a definition, the estimated value function will fit in with:
\begin{equation}
    \begin{aligned}
    V(s) \leq \frac{1 - \gW(p_{\gI_\pi^s}, p_g)}{1-\gamma}
    \end{aligned}
\end{equation}

where $\gW$ is some Wasserstein distance between marginal state distribution and goal distribution. Therefore, optimizing the Q value is equivalent to minimizing an upper bound for some types of the statistical distance between goal distribution and target distribution.

For the term related to policy in (\ref{elbo_final}), we can define the goal-conditioned policy distribution as:
\begin{equation}
    \pi(a^t | s^t, g) \propto \exp(Q(s^t, a^t)) 
\end{equation}
As a result, $\argmax_{\pi} \sum_{t=0}^{T-1} \log \pi_\theta(a^{t}|s^t, s^*, \gG) = \argmax_{\pi} \sum_{t=0}^{T-1} Q(s^t, a^t) $ 
However, the real $Q(s^t, a^t)$ is intractable, so we alternatively optimize the $\hat{Q}(s^t, a^t)$ at each time step. Next, we'll start to derive a bound between $\hat{Q}(s, \hat{\pi}(s))$ and $Q(s, \pi^*(s))$

\begin{proof}[Proof of Lemma~\ref{Value Bound}]
For simplicity, we denote the learned transition function $\hat{p}(s'|s, a) = p_\theta(s'|s, a, \gG)$, which is $\epsilon_m$-approximate dynamics, $\mathbb{D}_{\text{TV}}(\hat{p}, p) = \|\hat{p}(s'|s, a)-p(s'|s, a)\|_\infty \leq \epsilon_m$, 

    Firstly, we show by value iteration that the estimated value function $\hat{V}(s)$ will converge to $V(s)$:
    Assume exists $K>0$, s.t. $\forall k>K$, $\|\hat{p}(s'|,s,a)-p(s,|s,a)\|_\infty \leq \epsilon_m$
    \begin{equation}
        \begin{aligned}
            \hat{V}^{(k+1)}(s) = r(s, \pi(s)) + \gamma \sum_{s'} p(s'|s, \pi(s)) \hat{V}^{(k)}(s).
        \end{aligned}
    \end{equation}
        
    Given the result of the Bellman Contraction, 
    \begin{equation}
        \|\hat{V}^{(k+1)}(s) - V^{\pi^*}(s)\|_\infty \leq \gamma \|\hat{V}^{(k)}(s) - V^{\pi^*}(s)\|_\infty, \quad \lim_{k\to \infty} \|\hat{V}^{(k+1)}(s)-V^{\pi^*}(s)\|_\infty = 0.
    \end{equation}

    Based on the definition of greedy policy in planning: $\hat{\pi}(s) = \argmax_{a\in \gA} \hat{Q}(s, a)$, we can derive the inequality:
    \begin{equation}
    \begin{aligned}
         r(s, \hat{\pi}(s)) + \gamma \sum_{s'} \hat{p}(s'|s, \hat{\pi}(s)) \hat{V}(s')  & \geq r(s, \pi^*(s)) + \gamma  \sum_{s'} \hat{p}(s'|s, \pi^*(s)) \hat{V}(s')  \\
        \Longrightarrow\quad   r(s, \pi^*(s)) - r(s, \hat{\pi}(s))  &\leq \gamma \left[ \sum_{s'}  \hat{p}(s'|s, \hat{\pi}(s))\hat{V}(s') - \sum_{s'} \hat{p}(s'|s, \pi^*(s)) \hat{V}(s') \right]\\
        \stackrel{\|\hat{V}(s)-V^{\pi^*}(s)\|_\infty\to 0}{\Longrightarrow}  r(s, \pi^*(s)) - r(s, \hat{\pi}(s)) &\leq \gamma \left[\sum_{s'}  \hat{p}(s'|s, \hat{\pi}(s)) V^{\pi^*}(s') - \sum_{s'}\hat{p}(s'|s, \pi^*(s)) V^{\pi^*}(s') \right].
    \end{aligned}
    \end{equation}

    Then let $s$ be the state with the largest value error.
    \begin{equation}\label{eq:value_iter}
    \begin{aligned}
         V^{\pi^*}(s) - V^{\hat{\pi}}(s) = & r(s, \pi^*(s)) - r(s, \hat{\pi}(s)) \\
        & + \gamma \Big[\sum_{s'} p(s'|s,\pi^*(s))V^{\pi^*}(s') -\sum_{s'}  p(s'|s,\hat{\pi}(s))  V^{\hat{\pi}}(s') \Big] \\
        \leq & \gamma \sum_{s'}  \Big[\hat{p}(s'|s, \hat{\pi}(s)) V^{\pi^*}(s') -\hat{p}(s'|s, \pi^*(s)) V^{\pi^*}(s') \Big] \\
        & + \gamma \sum_{s'}\Big[p(s'|s,\pi^*(s)) V^{\pi^*}(s')-p(s'|s,\hat{\pi}(s)) V^{\hat{\pi}}(s') \Big] \\
        = & \gamma \sum_{s'} \Big[ p(s'|s,\pi^*(s))-\hat{p}(s'|s,\pi^*(s)) \Big] V^{\pi^*}(s') \\
        & - \gamma \sum_{s'}\Big[p(s'|s,\hat{\pi}(s))-\hat{p}(s'|s,\hat{\pi}(s))  \Big] V^{\pi^*}(s') \\
        & + \gamma \sum_{s'}p(s'|s,\hat{\pi}(s) ) \Big[V^{\pi^*}(s) - V^{\hat{\pi}}(s) \Big] \\
    \end{aligned}
    \end{equation}

    Since $r(s, g)\in [0, 1]$, the value function $V(s)\in [0, \frac{1}{1-\gamma}]$, also by $\|\hat{p}(s'|s,\hat{\pi}(s))-p(s'|s,\hat{\pi}(s))\|\leq \epsilon$, we have
    \begin{equation}
    \begin{aligned}
       V^{\pi^*}(s)- V^{\hat{\pi}}(s) & \leq \gamma \epsilon_m (V_{\text{max}}-V_{\text{min}}) + \gamma \sum_{s'}p(s'|s,\hat{\pi}(s) ) \Big[V^{\pi^*}(s) - V^{\hat{\pi}}(s) \Big] \\ 
       & = \frac{\gamma\epsilon_m}{1-\gamma} + \gamma \sum_{s'}p(s'|s,\hat{\pi}(s) ) \left[V^{\pi^*}(s')- V^{\hat{\pi}}(s') \right].
    \end{aligned}
    \end{equation}
    We already analyzed the state $s$ with the largest value error, and it's sufficient to show:
    \begin{equation}
        \begin{aligned}
        \|V^{\pi^*}(s) - V^{\hat{\pi}}(s) \|_\infty &\leq \frac{\gamma\epsilon_m}{1-\gamma} + \gamma \sum_{s'}p(s'|s,\hat{\pi}) \|V^{\hat{\pi}}(s) - V^{\pi^*}(s) \|_\infty \\
        & =  \frac{\gamma\epsilon_m}{1-\gamma} + \gamma \|V^{\pi^*}(s) - V^{\hat{\pi}}(s)  \|_\infty
        \end{aligned}
    \end{equation}
    By combining $\|V^{\pi^*}(s) - V^{\hat{\pi}}(s) \|_\infty$ on both sides, we have
    \begin{equation}
        \begin{aligned}
        \|V^{\pi^*}(s) - V^{\hat{\pi}}(s) \|_\infty \leq \frac{\gamma}{(1-\gamma)^2} \epsilon_m
        \end{aligned}
    \end{equation}
\end{proof}

\subsection{Causal Discovery}
\label{app:causal_discovery}
\subsubsection{Assumptions of Causality}
\label{app:causal_assumption}

\begin{assumption}[Markov property]
Given a DAG $\gG$ and a joint distribution $P_{\bm{X}}$, this distribution is said to satisfy 
\begin{itemize}[leftmargin=0.2in]
    \item (i) the global Markov property with respect to the DAG $\gG$ if 
        \begin{equation}
            \bm{A} \independent_{\gG} \bm{B} | \bm{C} \Rightarrow  \bm{A} \independent \bm{B} | \bm{C} 
        \end{equation}
        for all disjoint vertex sets $\bm{A}, \bm{B}, \bm{C}$. The symbol $independent_{\gG}$ denotes d-separation. 
    \item (ii) the local Markov property with respect to the DAG $\gG$ if each variable is independent of its non-descendants (without its parents) given its parents, and 
    \item (iii) the Markov factorization property with respect to the DAG $\gG$ if
        \begin{equation}
            p(\bm{x}) = p(x_1,\dots,x_d) = \prod_{j=1}^{d} p(x_j|\textbf{PA}^\gG (x_{j}))
        \end{equation}
        where we assume that $P_{\bm{X}}$ has a density $p$.
\end{itemize}
\end{assumption}

\begin{assumption}[Faithfulness]
Consider a distribution $P_{\bm{X}}$ and a DAG $\gG$, $P_{\bm{X}}$ is faithful to the DAG $\gG$ if we know
\begin{equation}
    \bm{A} \independent \bm{B} | \bm{C} \Rightarrow \bm{A} \independent_{\gG} \bm{B} | \bm{C}
\end{equation}
for all disjoint vertex sets $\bm{A}, \bm{B}, \bm{C}$.
\end{assumption}

\subsubsection{KL Divergence as Sparsity Regularization}
\label{app:approx_kld}

Similar to the assumption in Factorized MDP, the existence of a causal relationship between two arbitrary entities among $x$ is can also be treated as independent. Therefore, we construct the prior distribution $p(\mathcal{G})$ as independent Bernoulli Distribution in the transition causal graph.
\begin{equation}
p(\mathcal{G})=\prod_{i\in [M+N], j\in [M]} p(\mathcal{G}_{ij})=\prod_{i\in [M+N], j\in [M]} p_{ij}
\end{equation}

where $\mathcal{G}_{ij}$ represents the edge from i-th node in source node set $\mathcal{U}=\{\mathcal{A} \cup \mathcal{S}\}$ to the j-th node in target node set $\mathcal{V}=\{\mathcal{S}'\}$ in the bipartite transition causal graph.

On the other hand, for the variational posterior $q(\mathcal{G}|\tau)$, for the discovered transition causal graph, it needs to satisfy two constraints: (i) $q({\mathcal{G}|\tau})$ needs to be a DAG, denoted $\mathcal{Q}_{DAG}$, and more specifically, a bipartite graph. We denote such subset of DAG as $\mathcal{Q}_{Bi}$, (ii) $q({\mathcal{G}|\tau})$ needs to be as sparse as possible. 

Common score-based causal discovery works use two regularization terms, DAGNess and $\ell_1$ regularization to constrain the discovered causal graph in the constraint set, while in our work, we explicitly constrain the posterior variational distribution $ q(\mathcal{G}|\tau)\in \mathcal{Q}_{Bi}\subset \mathcal{Q}_{DAG}$. We then show in the following section that by defining a certain independent Bernoulli prior $p(\mathcal{G})$, the KL divergence between variational posterior $q(\mathcal{G}|\tau)$ and $p(\mathcal{G})$ can be equivalent to a sparsity regularization.

According to our constraint-based causal reasoning modules, $\mathcal{Q}_{Bi}$ consists of $M(M+N)$ independent binary classifiers (that form a DAG) parameterized by our kernel-based independent testing modules $\phi$, i.e. 
\begin{equation}
q_\phi(\mathcal{G}|\tau)=\prod_{i\in [M+N], j\in [M]} q_\phi(\mathcal{G}_{ij}|\tau)\triangleq \prod_{i\in [M+N], j\in [M]} q_{ij}
\end{equation}

\begin{proof}[Proof of Proposition~\ref{approx_kld}]
Let the prior $p_{ij}=\epsilon_{\gG}\in (0, \frac{1}{2}], \forall i\in [M+N], j\in [M]$, 
based on the definition above, the KL divergence term in (\ref{elbo_final}) can be expanded as follows:
\begin{equation}
\begin{aligned}
\mathbb{D}_{\text{KL}}(q_{\phi}(\mathcal{G}|\tau)\|p(\mathcal{G}))  &= \sum_{q\in \mathcal{Q}_{Bi}} \prod_{i,j}q_{ij} \log \frac{\prod_{i,j}q_{ij}}{\prod_{i,j}p_{ij}} =
\sum_{i, j} \left[q_{ij} \log \frac{q_{ij}}{p_{ij}} + (1-q_{ij}) \log \frac{1-q_{ij}}{1-p_{ij}} \right] \\
& = \sum_{i,j} \left[q_{ij} \log \frac{q_{ij}}{\epsilon_{\gG}} + (1-q_{ij}) \log \frac{1-q_{ij}}{1-\epsilon_{\gG}} \right] \\
& = \sum_{i,j} \left[q_{ij}\log q_{ij} + (1-q_{ij})\log (1-q_{ij}) -  q_{ij}\log\epsilon_{\gG} - (1-q_{ij})\log (1-\epsilon_{\gG}) \right]
\end{aligned}
\end{equation}
Since $q_{ij}\in \{0, 1\}$, $\lim_{q_{ij}\to 0} q_{ij}\log q_{ij}=\lim_{q_{ij}\to 1} (1-q_{ij})\log (1-q_{ij})=0, $
\begin{equation}
\begin{aligned}
    \mathbb{D}_{\text{KL}}(q_{\phi}(\mathcal{G}|\tau)\|p(\mathcal{G})) &= \sum_{i,j} \left[-q_{ij} \log (\epsilon_{\gG}) - (1-q_{ij}) \log (1-\epsilon_{\gG}) \right] \\
    &= \sum_{i,j} \left[ - \mathbb{I}(q_{ij}=1) \log \epsilon_{\gG} -\mathbb{I}(q_{ij}=0) \log (1-\epsilon_{\gG})\right] \\
    & = \sum_{i,j} \left[-(1-\mathbb{I}(q_{ij}=1)) \log (1-\epsilon_{\gG}) - \mathbb{I}(q_{ij}=1) \log \epsilon_{\gG}\right] \\
    & = \log \frac{1-\epsilon_{\gG}}{\epsilon_{\gG}} \sum_{i,j} \mathbb{I}(q_{ij}=1))   -\sum_{i,j} \log(1-\epsilon_{\gG}) \\
    & = \log \left(\frac{1-\epsilon_{\gG}}{\epsilon_{\gG}}\right) |q_\phi(\gG|\tau)|_1 + \text{const} \\
    & \triangleq \eta |q_\phi(\gG|\tau)|_1 + \text{const}
\end{aligned}
\end{equation}

Therefore, the KL divergence term is equivalent to an $\ell_1$ sparsity regularizer in score-based causal discovery~\cite{brouillard2020differentiable}. The strength of this regularizer $\eta=\log\left(\frac{1-\epsilon_{\gG}}{\epsilon_{\gG}}\right)\in [0, \infty)$.
The larger $\epsilon_{\gG}$ in prior Bernoulli distribution indicates the smaller strength of this sparsity regularizer (e.g. when $\epsilon_{\gG}=\frac{1}{2}, \eta=0$). 
In the implementation of data-efficient causal discovery, we adjust the parameter of the classifier to set the strength of this sparsity constraint.
\end{proof}

\subsubsection{Unique Identifiability of Causal Graph} \label{proof_uniqueness}
We construct our causal model based on the Factorized MDP in Assumption \ref{assumption_factorize}. According to the definition, the causal graph is a directed bipartite graph, with $s^{t}, a^{t}$ on the source side, and $s^{t+1}$ on the target side. For the theoretical analysis part in Section \ref{proof_uniqueness} and \ref{proof_monotonicity}, we denote $s^{t+1}$ as $s'$, $s_t$ as $s$, $a_t$ as $a$ and $\bm{x}=\{\gA \cup \gS \cup \gS'\}$, $\bm{x}\in \mathbb{R}^{2M+N}$ for simplicity.

\begin{definition}[Interventional Family $\gI$]
For any DAG $\gG$, we define the interventional family $\gI = (I_1, I_2, \cdots, I_K)$. Here $I_1:=\emptyset$ corresponds to the pure observational setting. The joint distribution for the interventional family can be rewritten as:
\begin{equation}\label{eq::intervention}
    p^{(k)}(x_1, \cdots, x_{[2M+N]}) = \prod_{j\notin I_k} p_j^{(1)} (x_j | \textbf{PA}^\gG(x_j)) \prod_{j\in I_k} p_j^{(k)} (x_j | \textbf{PA}^\gG(x_j))
\end{equation}
\end{definition}

\begin{definition}[]
For a specific DAG $\mathcal{G}$, we define $\mathcal{M}(\mathcal{G})$ to be the set of strictly positive densities $p: \mathbb{R}^{2|\mathcal{S}|+|\mathcal{A}|} \to \mathbb{R}$ which satisfies:
\begin{equation}\label{eq::positive}
    p(x_1, \cdots, x_{[2M+N]}) = \prod_{j\in [2M+N]} p_j (x_j | \textbf{PA}^\gG(x_j))
\end{equation}
where $\int_{\mathcal{X}_j} f_j(x_j|\textbf{PA}^\gG(x_j)) d x_j= 1$ for all $\textbf{PA}^\gG(x_j)\in \mathcal{X}_j$ and all $j\in [2M+N]$.
\end{definition}

\begin{definition}[]
For a specific DAG $\mathcal{G}$ and an interventional family $\gI$, 
we define 
\begin{equation}
    \mathcal{M}_\gI(\mathcal{G}):=\{[p^{(k)}]_{k\in [K]}\mid \forall k\in [K], p^{(k)}\in \gM(\gG), \forall j\notin I_k,  p_j^{(k)}(x|\textbf{PA}^\gG(x)) = p_j^{(1)}(x|\textbf{PA}^\gG(x))\}
\end{equation}
Such set of functions is conherent with condition of strictly positive densities in (\ref{eq::positive}) as well as factorization of interventional distribution in (\ref{eq::intervention}).
\end{definition}

\begin{definition}[$\gI$-Markov Equivalence Class, $\gI$-MEC]
Two DAGs $\gG_1$ and $\gG_2$ are $\gI$-Markov equivalent iff $\gM_\gI(\gG_1) = \gM_\gI(\gG_2)$. We denote by $\mathcal{I}-{MEC}(\gG_1)$ the set of all DAGs which are $\gI$-Markov equivalent to $\gG_1$, this is the $\gI$-Markov equivalence class of $\gG_1$.
\end{definition}

\begin{lemma}[Sufficient and Necessary Conditions for $\gI$-MEC, Yang el. al. \cite{Yang2018CharacterizingAL}]
    \label{lemma::mec}
    Suppose the interventional family $\gI$ is such that $\gI_1 := \emptyset$. Two DAGs $\gG_1$ and $\gG_2$ are $\gI$-Markov equivalent iff their I-DAGs $\gG_{\gI_1}$ and $\gG_{\gI_2}$ share the same skeleton and v-structures.
\end{lemma}

\begin{proof}[Proof of Proposition \ref{uniqueness}]
In the bipartite graph $(\gU, \gV, E)$, for the discovered graph $\hat{\gG}$ that is in the $\gI$-Markov equivalence class of the ground truth causal graph, $\hat{\gG}$ is unique.

Based on the Lemma~\ref{lemma::mec}, all possible $\hat{\gG}$ that are $\gI$-Markov equivalent will share an identical skeleton with $G^*$, so we consider only graphs obtained by reversing edges in $\hat{\gG}$.   

Due to the bipartite nature of the transition causal graph defined in Definition \ref{transition_cg}, for all the v-structured colliders $c\in \mathcal{C}$, we know that $c\in \gS'$, therefore, reversing any edge of $\hat{\gG}$ will harm the immorality of $\hat{\gG}$, and the new graph will no longer be an $\gI$-MEC to $\gG^*$ Therefore, $\hat{\gG}$ is the only graph in the $\gI$-MEC of $\gG^*$,  i.e. $\hat{\gG}=\gG^*$.
\end{proof}

\subsubsection{Causal Discovery Benefits from Policy Learning}
In this section, we would like to show how the learned GCRL model could aid the performance of causal discovery.
Before we put the formal proof, we first list several assumptions which is quite common in causal discovery literatures~\cite{pmlr-v119-addanki20a}.

\begin{assumption}[Oracle Conditional Independence Test]
\label{app:oracle}
The conditional independent test could tell the independence between any two variables in the causal graph.
\end{assumption}

\begin{proof}[Proof of Lemma~\ref{causal_discovery_quality}]
Given the oracle conditional independence test in appendix~\ref{app:oracle}, if the state distribution covers all the support of goal distribution, with abundant actions from action space, we can cover all the connections between the current state and the next states.

When $\mathbb{D}_{\text{TV}}(p_{\mathcal{I}_\pi^s}, p_g)< \epsilon_g$, it is sufficient to derive that $p_{\mathcal{I}_\pi^s}(s)>0, \forall s\in \text{Supp}_g$, where $\text{Supp}_g$ is the support set of the goal distribution $p_g$. 

We then discuss the three possible circumstances under such condition: 
\begin{itemize}
    \item Case 1: Both source state and target state distribution in the buffer are fully supported on $\text{Supp}_g$. In this case, given our Assumption~\ref{app:oracle} and abundant samples, our causal discovery $q_\phi(\mathcal{G}|\tau)$ will correctly classify all the edges in the transition causal graph.
    \item Case 2: Only the target state distribution is supported on $\text{Supp}_g$,  while the source side leaves away from the goal node. In this case, the independent tests may not be able to distinguish the (in)dependence relationship between goal nodes and other state nodes,  $\text{SHD}(\hat{\mathcal{G}}, \mathcal{G}^*)\leq |\mathcal{S}|-1$.
    \item Case 3: Only source state is supported on $\text{Supp}_g$, while the target side leaves away from the goal node. This case corresponds to the case where some initial states hit the goal, while the learned transition model and policy fail to guide the future states to the goal. The causal discovery model $\phi$ may falsely classify the edges from all the source states (except the source goal state) towards the target goal states. Thus $\text{SHD}(\hat{\mathcal{G}}, \mathcal{G}^*)\leq |\mathcal{S}|-1$.
\end{itemize}

In conclusion, for all the three cases that satisfy  $\mathbb{D}_{\text{TV}}(p_{\mathcal{I}_\pi^s}, p_g)\leq \epsilon_g$, we have
\begin{equation}
    \max_{\hat{\mathcal{G}} \sim q_\phi(\cdot|\tau) } \Big[\text{SHD}(\hat{\mathcal{G}}, \mathcal{G}^*)\Big]\leq |\mathcal{S}|-1
\end{equation}
thus 
\begin{equation}
\mathbb{E}_{\hat{\mathcal{G}} \sim q_\phi(\cdot|\tau) } \Big[\text{SHD}(\hat{\mathcal{G}}, \mathcal{G}^*)\Big] \leq \max_{\hat{\mathcal{G}} \sim q_\phi(\cdot|\tau) } \Big[\text{SHD}(\hat{\mathcal{G}}, \mathcal{G}^*)\Big]\leq |\mathcal{S}|-1
\end{equation}

\end{proof}

\subsection{Overall Performance Guarantee of Iterative Optimization}
\label{app:convergence}

Based on all the derivation from previous sections, we finally give out the proof of the overall performance of our proposed iterative optimization in \textit{GRADER}.

\begin{proof}[Proof of Theorem~\ref{convergence}]
Let $d_{max}=\max_{s_1, s_2\in \mathcal{S}} \|s_1-s_2\|^2, d_\theta=\|\hat{s}'(\theta)-s'\|^2$,  then the log likelihood term becomes
\begin{equation}
    p_\theta(s'|s,a)\propto \exp(d_{max}-d_\theta)
\end{equation}

In the model learning part, since we take the log space, we have $\log p_\theta(s'|s, a)=(d_{max}-d_\theta) - C$. 
We neglect the constant term $C$ when deriving the bound.
Without the loss of generality, we set $p_\theta(s'|s,a)= \exp(d_{max}-d_\theta), \log p_\theta(s'|s, a) = d_{max}-d_\theta$ in (\ref{elbo_final}). As $d_{max}-d_\theta \geq 0$, we have the Lipchitz $L\leq 1$ of log function, 
\begin{equation}
    \begin{aligned}
        \Big\|\log \hat{p}(s'|s, a) - \log p(s'|s, a) \Big\|_\infty & \leq  L \Big\|\hat{p}(s'|s,a) - p(s'|s,a)\Big\|_\infty \\
        & \leq \Big\|\hat{p}(s'|s,a) - p(s'|s,a)\Big\|_\infty \leq \epsilon_m
    \end{aligned}
\end{equation}

Based on Lemma~\ref{Value Bound} that is derived in Appendix~\ref{app:planning}, we have the policy learning term 

\begin{equation}
    \begin{aligned}
    \Big\|\log \hat{\pi}(a|s, g) - \log \pi^*(a|s, g)] \Big\|_\infty & = \Big\|\hat{Q} (s, \hat{\pi}(s)) -  Q (s, \pi^*(s)) \Big\|_\infty \\
    & = \Big\|Q (s, \hat{\pi}(s)) -  Q (s, \pi^*(s)) \Big\|_\infty \\
    & = \Big\|V^{\hat{\pi}} (s) -  V^{\pi^*} (s) \Big\|_\infty \\
    & \leq \frac{\gamma}{(1-\gamma)^2} \epsilon_m
    \end{aligned}
\end{equation}

For the KL divergence term, if the goal distribution satisfies $\epsilon_g > \frac{\gamma}{1-\gamma}\epsilon_m$, the following conditions hold:

\begin{equation}
    V^{\hat{\pi}}(s) > V^{\pi^*}(s) - \frac{\gamma}{(1-\gamma)^2} \epsilon_m
\end{equation}
According to Lemma~\ref{TV_distance} and the condition that $V(s)\in [0, \frac{1}{1-\gamma}]$,
\begin{equation}
\begin{aligned}
    \mathbb{D}_{\text{TV}} (p_{\gI_\pi^s}, p_g)  & \leq 1 - (1-\gamma) V^{\hat{\pi}}(s) \\
    & < 1 - (1-\gamma) V^{\pi^*}(s) + \frac{\gamma}{1-\gamma}\epsilon_m \\
    & = (1-\gamma^{t^*-1}) + \epsilon_g \stackrel{t^*=1}{=} \epsilon_g
\end{aligned}
\end{equation}
where $t^*$ is the shortest time step to reach the goal. Here we assume $t^*=1$ for optimal policy in the theoretical design part, while in practice, the bound may get loosened when larger $t^*$ or smaller $\gamma$.

Since $\mathbb{D}_{\text{TV}} (p_{\gI_\pi^s}, p_g)\leq \epsilon_g$, according to Lemma~\ref{causal_discovery_quality} proved in Appendix~\ref{app:causal_discovery}, we have
\begin{equation}
    \begin{aligned}
    \Big\|\mathbb{D}_{\text{KL}}(q_\phi|| p)-\mathbb{D}_{\text{KL}}(q_\phi^*|| p) \Big\|_\infty & = \Big\|\log \left(\frac{1-\epsilon_\mathcal{G}}{\epsilon_\mathcal{G}} \right) \|q_\phi(\mathcal{G}|\tau)\|_1-\log \left(\frac{1-\epsilon_\mathcal{G}}{\epsilon_\mathcal{G}} \right) \|q_\phi^*(\mathcal{G}|\tau)\|_1 \Big\|_\infty \\
    & = \log \left(\frac{1-\epsilon_\mathcal{G}}{\epsilon_\mathcal{G}} \right) \Big\| \|q_\phi(\mathcal{G}|\tau)\|_1-\|q_\phi^*(\mathcal{G}|\tau)\|_1 \Big\|_\infty \\
    & \leq \log \left(\frac{1-\epsilon_\mathcal{G}}{\epsilon_\mathcal{G}} \right) \Big \|q_\phi(\mathcal{G}|\tau)-q_\phi^*(\mathcal{G}|\tau)\Big \|_\infty  \\
    & = \log \left(\frac{1-\epsilon_\mathcal{G}}{\epsilon_\mathcal{G}} \right) \max_{\gG} \Big[\text{SHD}(\mathcal{G}, \mathcal{G}^*)\Big]\\
    & \leq \log \left(\frac{1-\epsilon_\mathcal{G}}{\epsilon_\mathcal{G}} \right) \Big(|\mathcal{S}|-1 \Big)
    \end{aligned}
\end{equation}

Finally, we can derive the overall performance guarantee as follows:

\begin{equation}
    \begin{aligned}
    \|\gJ^*(\theta, \phi) - \hat{\gJ}(\hat{\theta}, \hat{\phi})\|_\infty  = &\   \Big\| \sum_{t=0}^{T-1}\Big\{ \Big[\log \hat{p}(s^{t+1}|s^t, a^t) - \log p(s^{t+1}|s^t, a^t)\Big] \\
    & + \Big[\log \hat{\pi}(a^{t}|s^t, s^*) - \log \pi^*(a^{t}|s^t, s^*)\Big]  \Big\} + \Big[\mathbb{D}_{\text{KL}} (\hat{q}_\phi || p) -  \mathbb{D}_{\text{KL}} (q^*_\phi || p)\Big]  \Big\|_\infty \\ 
    \leq &\  \sum_{t=0}^{T-1}\Big\{ \Big\| \log \hat{p}(s^{t+1}|s^t, a^t) - \log p(s^{t+1}|s^t, a^t)\Big\|_\infty \\
    & + \Big\| \log \hat{\pi}(a^{t}|s^t, s^*) - \log \pi^*(a^{t}|s^t, s^*)\Big\|_\infty \Big\} + \Big\|  \mathbb{D}_{\text{KL}} (\hat{q}_\phi || p) -  \mathbb{D}_{\text{KL}} (q^*_\phi || p)\Big\|_\infty \\
    \leq & \  \sum_{t=0}^{T-1} \Big(\epsilon_m + \frac{\gamma}{(1-\gamma)^2}\epsilon_m \Big) + \log \left(\frac{1-\epsilon_\mathcal{G}}{\epsilon_\mathcal{G}} \right) \Big(|\mathcal{S}|-1 \Big) \\
    =&\ \Big[ 1 + \frac{\gamma}{(1-\gamma)^2} \Big] \epsilon_m T + \log \left(\frac{1-\epsilon_\mathcal{G}}{\epsilon_\mathcal{G}} \right) \Big(|\mathcal{S}|-1 \Big)
    \end{aligned}
\end{equation}
\end{proof}

\section{Additional Experiment}

\subsection{Overall Performance}
\label{app:more_results}

The overall performance results corresponding to the Table~\ref{tab:overall} for \textit{Stack} and \textit{Unlock} environments are shown in Figure~\ref{fig:app_reward_stack} and Figure~\ref{fig:app_reward_unlock}.

\begin{figure}[h]
    \centering
    \includegraphics[width=1.0\textwidth]{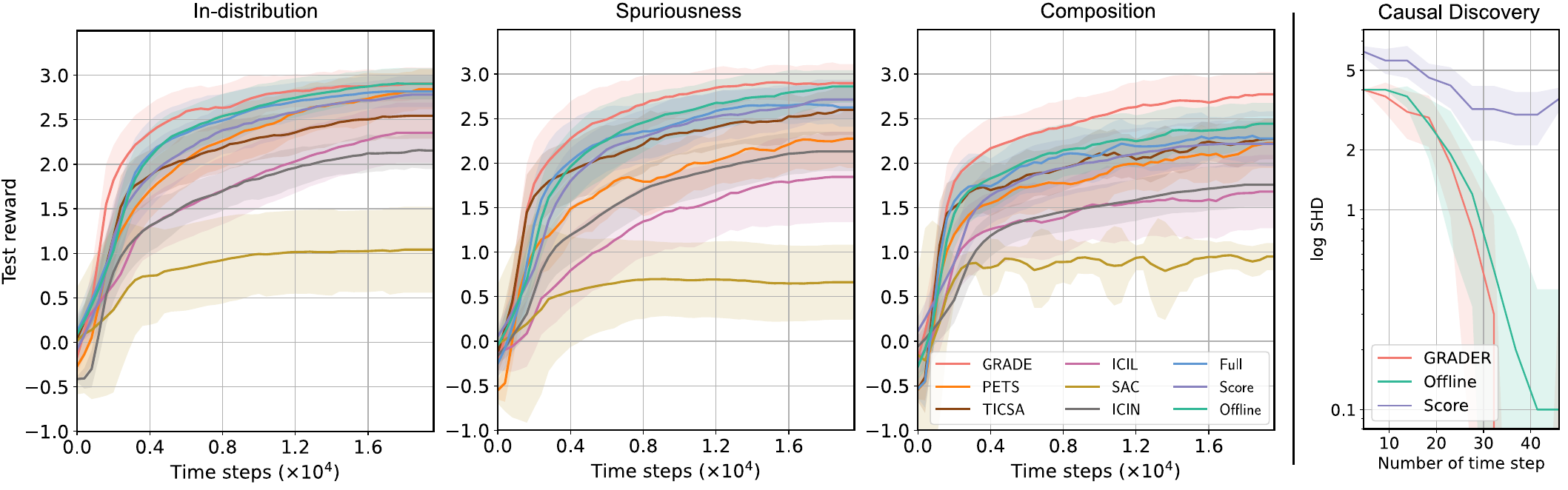}
    \caption{The testing reward and causal discovery results of \textit{Stack} environment.}
    \label{fig:app_reward_stack}
\end{figure}

\begin{figure}[h]
    \centering
    \includegraphics[width=1.0\textwidth]{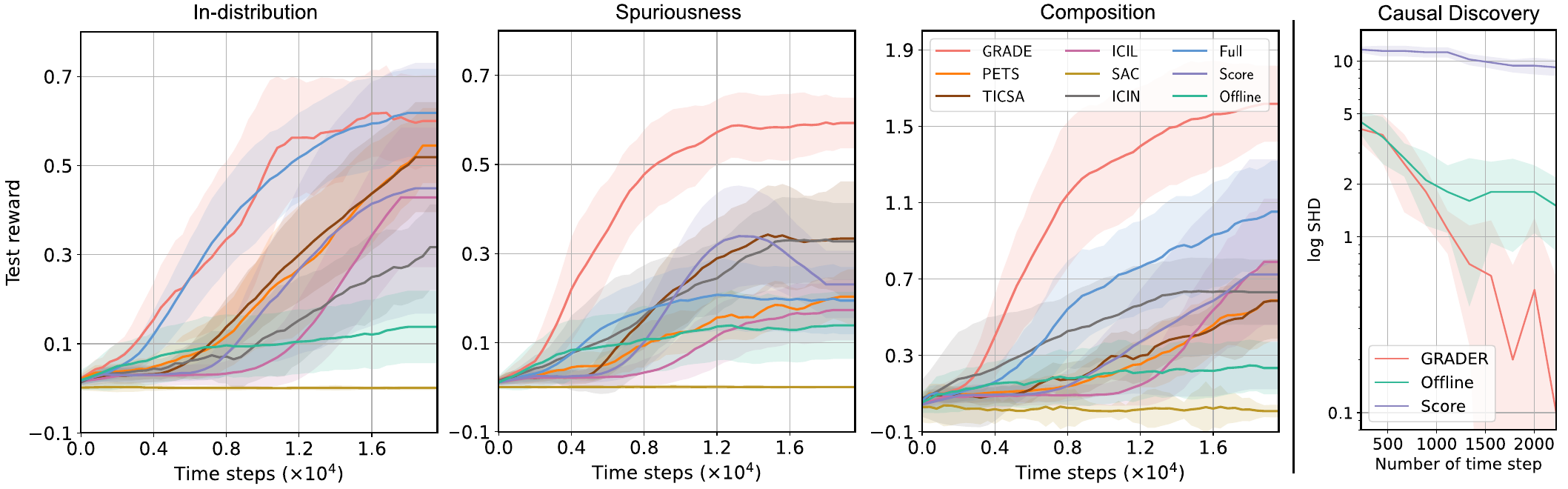}
    \caption{The testing reward and causal discovery results of \textit{Unlock} environment.}
    \label{fig:app_reward_unlock}
\end{figure}

In all \textit{Stack} experiments, we find that the advantage of GRADER over other methods is small. The reason is that this task is simple and the true causal graph only contains 7 nodes as shown in Figure~\ref{app:causal_graph}. 
Due to the simple causal graph, even the Offline random policy can obtain the true causal graph, thus there is almost no difference between the discovery efficiency between GRADER and Offline as shown in the right part of Figure~\ref{fig:app_reward_stack}.

In the \textit{Unlock-I} experiment, there is no gap between GRADER and Full, which means the causal graph may not have many contributions to solving this task. However, there are large gaps in \textit{Unlock-S} and \textit{Unlock-C} settings since indicating that the causal graph helps the model obtain better generalizable performance. 
As for the Offline method, since the causal graph is wrongly discovered, the performance is bad in all three settings.

\begin{figure}[t]
    \centering
    \includegraphics[width=1.0\textwidth]{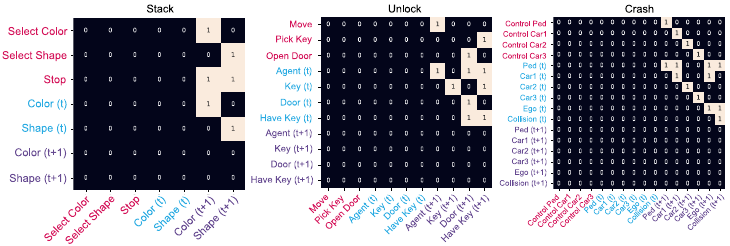}
    \caption{Discovered causal graphs of three environments. Color meaning: \color{magenta}{Action}\color{black}{,} \color{cyan}{State}\color{black}{,} \color{violet}{Next state}.}
    \label{fig:app_causal_graph}
\end{figure}

\begin{wrapfigure}{r}{0.35\textwidth}
\vspace{-10mm}
    \centering
    \includegraphics[width=0.35\textwidth]{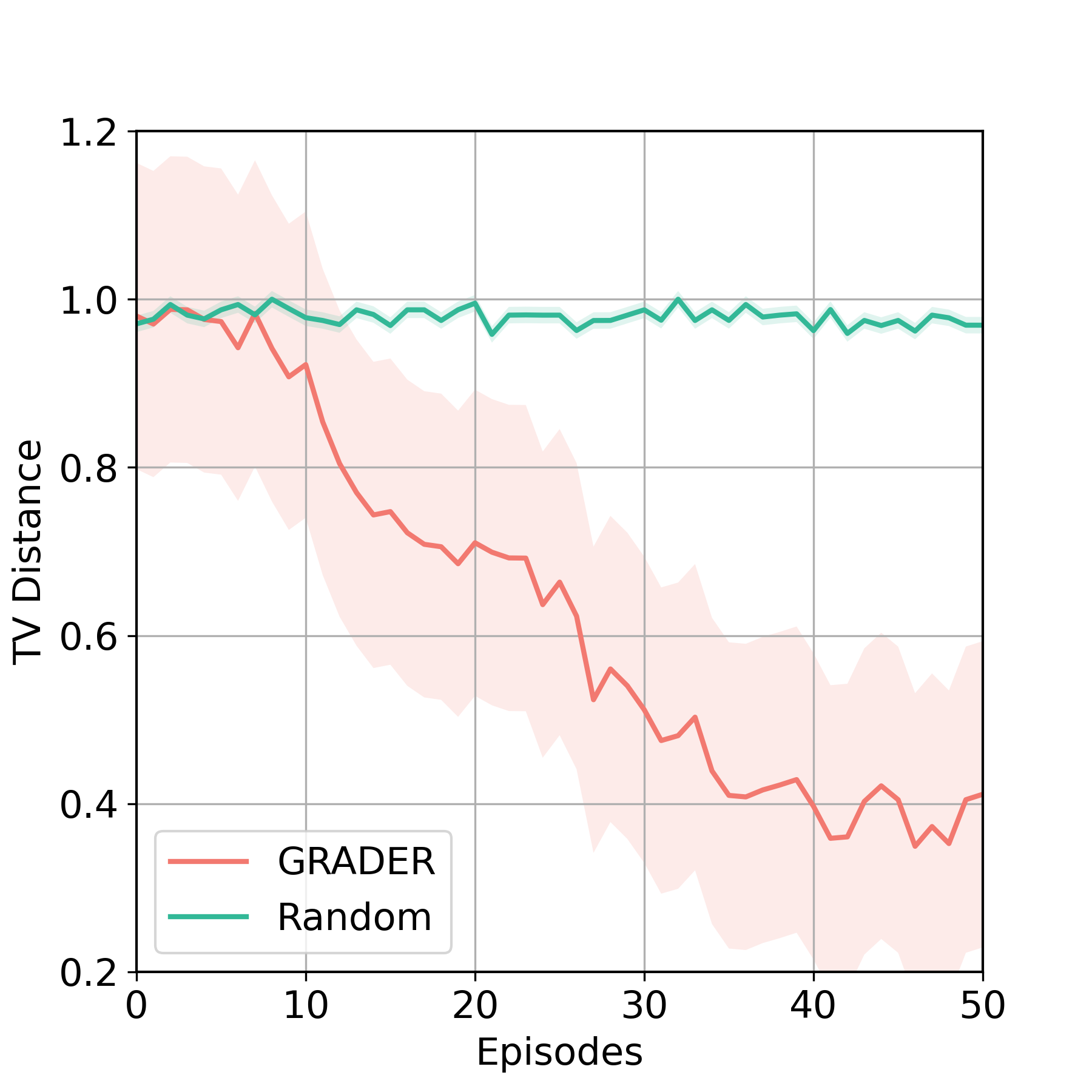}
    \caption{TV distance between goal and state distributions.}
    \label{fig:app_tv_distance}
\vspace{-11mm}
\end{wrapfigure}

\subsection{Causal Graph Analysis}
\label{app:causal_graph}

Since the environments we designed have clear and explicit causality, we can get the true causal graph with human analysis. 
We plot the true causal graphs corresponding to the three environments in Figure~\ref{fig:app_causal_graph}, where the semantic meanings of all nodes are explained in Appendix~\ref{detail_env}. 
We observe that the causal graphs are sparse with very few edges, indicating that non-causal methods that use the full graph may import redundant or even wrong information.

\subsection{Distance between Goal and State Distribution}

In Figure~\ref{fig:app_tv_distance}, we empirically show the upper bound proved in (\ref{eq:distance_bound}), which describes the TV distance between the goal distribution and the state distribution collected from the GRADER policy.
We use 10 trails and plot the mean and standard derivation of the distance.
We observe that the distance becomes smaller as the policy gets better in GRADER.
This supports our statement that the planning module helps to collect better data samples, which will be used in the causal discovery module.
We also plot the distance with a random policy, which is always large since the goal is not easy to be achieved by random actions.

\begin{figure}[t]
    \centering
    \includegraphics[width=1.0\textwidth]{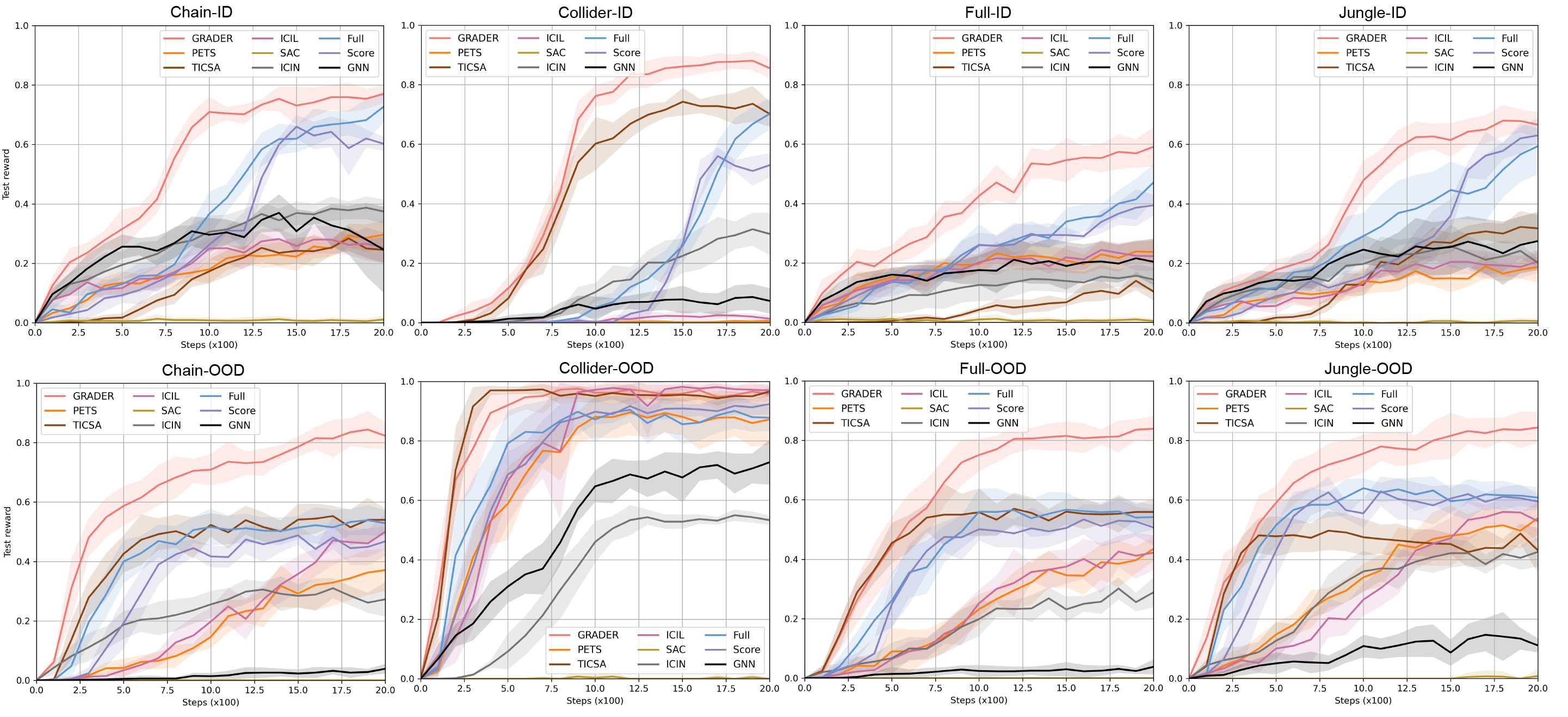}
    \caption{Reward of Chemistry environment under ID and OOD setting.}
    \label{fig:chemistry}
\vspace{4mm}
\end{figure}

\subsection{Task Performance of Chemistry Experiment}
\label{app:chemistry_experiment}

In the main context, we only show the discovery results of the Chemistry experiment. We provide the results of task performance in this section.
The downstream task is to change the color of nodes to match given colors within maximum steps ($T=10$). A reward $r=1$ is received if all colors are matched. Results are reported with $200$ episodes. We use planning horizon $H=5$. 
We provide the RL downstream task results in Table~\ref{tab:chemistry_id} (ID setting) and Table~\ref{tab:chemistry_ood} (OOD setting). The testing reward is shown in Figure~\ref{fig:chemistry}. 
The graphs in the ID setting have 10 nodes while those in the OOD setting have 5 nodes. In the ID setting, we randomly sample the target colors in the goal. In the OOD setting, we set the target colors of all nodes to the same color during the training to create spurious correlations, then randomly set the target colors during testing.

\begin{table}[ht]
\caption{Success rate (\%) for Chemistry environment (ID). \highest{Bold} font means the best.}
\label{tab:chemistry_id}
\centering
\small{
  \begin{tabular}{l|p{0.9cm} p{0.9cm} p{0.9cm} p{0.9cm} p{0.9cm} p{1.2cm} | p{0.9cm} p{0.9cm} p{0.9cm}}
    \toprule
    \small{Env}         & \small{SAC}                   & \small{ICIN}                       & \small{PETS}                      & \small{TICSA}                 
                        & \small{ICIL}                  & \small{GNN}  & \scriptsize{GRADER}                      & \small{Score}                     & \small{Full}                \\
    \midrule 
    \small{Collider}    & 0.0$\pm$0.0 & 29.8$\pm$7.2 & 0.6$\pm$0.8  & 70.1$\pm$4.2 & 1.3$\pm$1.3  &7.3$\pm$4.3 & \textbf{85.5$\pm$3.4} & 53.0$\pm$4.1 & 70.3$\pm$5.3 \\
    \small{Chain}       & 1.1$\pm$1.3 & 37.5$\pm$4.0 & 29.6$\pm$4.9 & 24.5$\pm$4.3 & 25.3$\pm$5.1 & 24.6$\pm$14.5 & \textbf{77.0$\pm$3.2} & 60.2$\pm$2.4 & 72.7$\pm$5.3  \\
    \small{Jungle}      & 0.6$\pm$0.8 & 20.2$\pm$1.5 & 18.8$\pm$5.2 & 31.8$\pm$4.5 & 20.6$\pm$3.9 & 27.5$\pm$9.8  & \textbf{69.6$\pm$4.3} & 63.0$\pm$2.3 & 59.4$\pm$9.5  \\
    \small{Full}        & 0.5$\pm$0.8 & 4.5$\pm$4.0  & 23.7$\pm$4.3 & 10.4$\pm$3.0 & 22.3$\pm$5.1 & 20.4$\pm$7.8  & \textbf{59.1$\pm$6.6} & 39.4$\pm$3.5 & 47.1$\pm$6.1  \\
    \bottomrule
  \end{tabular}
}
\end{table}

\begin{table}[ht]
\caption{Success rate (\%) for Chemistry environment (OOD). \highest{Bold} font means the best.}
\label{tab:chemistry_ood}
\centering
\small{
  \begin{tabular}{l|p{0.9cm} p{0.9cm} p{0.9cm} p{0.9cm} p{0.9cm} p{1.2cm} | p{0.9cm} p{0.9cm} p{0.9cm}}
    \toprule
    \small{Env}         & \small{SAC}                   & \small{ICIN}                       & \small{PETS}                      & \small{TICSA}                 
                        & \small{ICIL}                  & \small{GNN}  & \scriptsize{GRADER}                      & \small{Score}                     & \small{Full}                \\
    \midrule 
    \small{Collider}   & 0.0$\pm$0.0 & 53.3$\pm$1.6 &  87.2$\pm$8.5 &  96.6$\pm$1.4 &  97.0$\pm$2.0 &  72.8$\pm$7.5 & \textbf{95.8$\pm$2.6}     & 92.4$\pm$3.5 & 87.8$\pm$4.4 \\
    \small{Chain}    & 0.0$\pm$0.0 & 27.3$\pm$5.9 & 37.1$\pm$7.0 & 54.0$\pm$3.8 & 50.0$\pm$5.8  & 3.9$\pm$1.6  &  \textbf{82.3$\pm$4.5} & 46.8$\pm$5.0 & 52.9$\pm$4.3 \\
    \small{Jungle}    & 0.8$\pm$2.4 & 42.6$\pm$4.9 & 53.9$\pm$5.5 & 43.1$\pm$7.5 & 52.9$\pm$7.0  & 11.1$\pm$2.4 &  \textbf{84.4$\pm$5.1} & 59.5$\pm$2.7 & 60.8$\pm$3.5 \\
    \small{Full}   & 0.0$\pm$0.0 & 28.9$\pm$5.0 & 43.5$\pm$4.1 & 55.9$\pm$4.5 & 42.2$\pm$5.9  & 3.8$\pm$2.5 & \textbf{83.9$\pm$4.4} &  50.7$\pm$6.0 & 54.2$\pm$4.1  \\
    \bottomrule
  \end{tabular}
}
\vspace{-4mm}
\end{table}

\section{Additional Information}

\subsection{Details about Conditional Independence Test}

In Algorithm~\ref{algorithm1}, we describe the discovery of causal graph with edge inference $e_{ij} \leftarrow q_{\phi}(\cdot|\mathcal{B}, \eta)$ implemented by conditional independent test. We ignore the details about the test process in the main context and thus provide more details in this section.

For discrete variables, we use Pearson's chi-square test\footnote{https://en.wikipedia.org/wiki/Pearson\%27s\_chi-squared\_test}, which is a statistical test applied to sets of categorical data to evaluate how likely it is that any observed difference between the sets arose by chance. In our experiment, we use the implementation provided by package Scipy \footnote{https://docs.scipy.org/doc/scipy/reference/generated/scipy.stats.chi2\_contingency.html}.

We first define the null hypothesis, which is true when two random variables are statistically independent. These two variables have samples stored in a contingency table $O$, which has $c$ columns and $r$ rows. Then, the ``theoretical frequency'' for a cell is:
\begin{equation}
    E_{ij} = N_{p_{i\cdot}p_{\cdot j}},\ \ p_{i\cdot} = \sum_{j=1}^c \frac{O_{i,j}}{N},\ \ p_{i\cdot} = \sum_{i=1}^r \frac{O_{i,j}}{N}
\end{equation}
where $N$ is the total sample size in the table, $O_{i,j}$ is the sample size of cell $(i, j)$. Then, we can calculate the value of the test statistic:
\begin{equation}
    \chi^2 = \sum_{i=1}^r \sum_{j=1}^c \frac{(O_{i,j}-E_{i,j})^2}{E_{i,j}}
\end{equation}
Now, we can obtain a p-value (falls in $[0, 1] $) that indicates the significance of this statistic follows the $\chi^2$ distribution from chi-square probability\footnote{https://people.richland.edu/james/lecture/m170/tbl-chi.html}. We compare this p-value with a threshold $\eta$ and reject the null hypothesis if the p-value is smaller than $\eta$. Therefore, the larger we set $\eta$, the more likely we find the two variables are dependent. This testing process is summarized in Algorithm~\ref{algorithm2}.

If the two variables are continuous, we cannot use the above statistical test anymore. We turn to a more advanced test method proposed in~\cite{chalupka2018fast}. The general idea is that if $P(X|Y,Z) = P(X,Y)$, $Z$ is not useful as a feature to predict $X$. To achieve this, the authors propose to use decision tree regression to predict Y using both X and Z, and also using Z only.

\begin{algorithm}[H]
\caption{Independence Test for Discrete Variables.}
\label{algorithm2}
\KwIn{A contingency table $O$ with samples for two variables $X$ and $Y$.}

Define Null hypothesis: $X$ and $Y$ are independent. \\
Calculate $p_{i\cdot} = \sum_{j=1}^c \frac{O_{i,j}}{N}$ and $p_{i\cdot} = \sum_{i=1}^r \frac{O_{i,j}}{N}$ \\
Calculate expected frequencies $E_{ij} = N_{p_{i\cdot}p_{\cdot j}}$ \\
Calculate the chi-square statistic $\chi^2 = \sum_{i=1}^r \sum_{j=1}^c \frac{(O_{i,j}-E_{i,j})^2}{E_{i,j}}$ \\
Obtain p-value $p$ from chi-square probability \\
\If{p < $\eta$}
{
Reject the Null hypothesis, i.e., $X$ and $Y$ are dependent.
}
\end{algorithm}

\subsection{Experiment Details}

\subsubsection{Environment Design}
\label{detail_env}

More details about the design of the environments are summarized below:

\textbf{Stack}: 
Manipulation is important for house-holding and factory assembly.
Sometimes, the color of the object is not relevant to the task but may leak information by sharing spuriousness with the task. Also, the goal could compose several previously seen goals such as repeating similar actions. 
Totally, we have 5 different shapes and 5 different colors and the goals are some combinations of the colors and shapes. 
At each step, the agent can either stack an object with a chosen color and shape or stop stacking. 
The state is the colors and shapes of all current objects.
The agent receives a positive reward when the goal is achieved and a punishment if it stacks a new object.

\textbf{Unlock}: 
Collecting specific objects to fulfill required conditions is useful for mobile robots.
The causality in this environment exists between the key and the door. 
The action contains six operations, including four-direction movements (Move), pick key (Pick Key), and open door (Open Door).
The state is the position of the agent, the position of the key, and the status of the door.
In the first generalization setting, we intentionally create a spurious correlation between the position of the key and the door. If the agent figures out that key can open the door no matter what its position is, it will ignore the spurious correlation.
In the second generalization setting, we increase the number of door from one to two. This setting contains two same sub-tasks and can be used to test the compositional generalization.
   
\textbf{Crash}: 
The causality in this environment mainly exists between the pedestrian (Ped), the ego vehicle (Ego), and another vehicle (Car 1)~\cite{ding2022survey}. The collision between Ped and Ego only happens when the view of Ego is blocked by Car 1. To make this happen, we design a rule-based AV, which will brake if it detects any obstacles within a certain distance. Therefore, if the pedestrian directly hits the AV, the AV will stop and the crash will not happen.
To make this task difficult, we also place two other vehicles (Car 2 and Car 3) on the scene but they will not interrupt the crash scenario. 
The agent can control the acceleration and steering of Ped, Car 1, Car 2, and Car 3.
The state is the position and velocity of all objects plus the status of whether a collision happens.
To create a spurious correlation, we fix the initial distance between Ego and Ped to a constant since this creates a shortcut for the feature extractor. However, remembering this distance is not enough since we change the initial distance in the testing stage.

\textbf{Chemistry}: Please refer to \cite{ke2021systematic} for more details.

\begin{table}[h]
\renewcommand\arraystretch{1.1}
\caption{Environment configurations used in experiments}
\label{app:parameters}
\centering
  \begin{tabular}{c|c|c|c|c}
    \toprule
    Parameters                  & Stack   & Unlock  & Crash & Chemistry  \\
    \midrule
    Max step size               & 5       & 15     & 30  & 10 \\
    State dimension             & 50      & 110    & 22  & 100 \\
    Action dimension            & 12      & 8      & 8   &  100 \\
    Action type                 & Discrete    & Discrete      & Continuous  & Discrete  \\
    \bottomrule
  \end{tabular}
\end{table}

\begin{table}[h]
\renewcommand\arraystretch{1.1}
\caption{Hyper-parameters of models used in experiments}
\label{app:parameters}
\centering
\begin{threeparttable}
  \begin{tabular}{c|c|c|c|c|c}
    \toprule
    \multirow{2}{*}{Models} & \multirow{2}{*}{Parameters} & \multicolumn{3}{c}{Environment} \\
    \cline{3-5}
              &  & Stack   & Unlock  & Crash & Chemistry \\
    \midrule
    \multirow{10}{*}{GRADER} 
    & Learning rate               & 0.001   & 0.001  & 0.0001 & 0.001  \\
    & Size of buffer $\gB$        & 4000   & 10000  & 10000   & 4000 \\
    & Epoch per iteration         & 20      & 5     & 10     & 20  \\
    & Batch size                  & 256   & 256  & 256       & 256\\
    & Planning horizon $H$        & 5      & 10    & 20    & 5 \\
    & Planning population         & 500   & 100  & 1000    & 700  \\
    & Reward discount $\gamma$    & 0.99   & 0.99  & 0.99  & 0.99  \\
    & $\epsilon$-greedy ratio     & 0.4   & 0.4  & 0.5     & 0.5  \\
    & Causal Discovery $\eta$     & 0.01   & 0.01  & 0.01  & 0.01  \\
    & GRU hiddens                 & 32   & 64  & 128       & 32 \\
    \midrule
    \multirow{3}{*}{PETS\tnote{*}}         
    & MLP hiddens & 32 & 64 & 128  & 32 \\
    & MLP layers   & 2     & 2  & 2  & 2 \\
    & Ensemble number & 5 & 5 & 5  & 5\\
    \midrule
    \multirow{5}{*}{TICSA\tnote{*}}  
    & Size of buffer $\gB$           & 20000   & 400000  & 40000 & 20000  \\
    & Pretrain buffer                & 200   & 2000  & 5000 & 200  \\
    & Initialized mask coef.         & 1.0   & 1.0  & 1.0  & 1.0 \\
    & MLP hiddens                     & 32   & 64  & 128   & 32  \\
    & Sparsity regularizer           & 0.5   & 1.0  & 0.2  & 0.5 \\
    \midrule
    \multirow{5}{*}{ICIL\tnote{*}} 
    & Size of buffer $\gB$           & 20000   & 400000  & 40000 & 20000   \\
    & Learning rate of MINE         & 0.0001 & 0.0001 & 0.0001   & 0.0001\\
    & MLP hiddens                   & 32   & 64  & 128  & 32 \\
    & MINE hiddens                    & 32   & 64  & 128  & 32 \\
    & Env. Numbers                 & 5 & 3 & 3 & 5 \\
    \midrule
    \multirow{8}{*}{SAC}         
    & Learning rate               & 0.001   & 0.001  & 0.0001  & 0.001  \\
    & Size of buffer $\gB$        & 4000   & 10000  & 10000    & 4000 \\
    & Update step $\tau$          & 0.005   & 0.005  & 0.0001  & 0.005 \\
    & Update iteration            & 3      & 3  & 3 & 3  \\
    & Entropy $\alpha$            & 0.2   & 0.2  & 0.2  & 0.2 \\
    & Batch size                  & 256    & 256  & 256  & 256 \\
    & Reward discount $\gamma$    & 0.99   & 0.99  & 0.99   & 0.99 \\
    & MLP hiddens                 & 64     & 128  & 256  & 64 \\
    \midrule
    \multirow{5}{*}{ICIN} 
    & Learning rate               & 0.001   & 0.001  & 0.0001  & 0.001  \\
    & Size of buffer $\gB$        & 4000   & 10000  & 10000    & 4000\\
    & Batch size                  & 256   & 256  & 256  & 256 \\
    & MLP hiddens                 & 64     & 128  & 256  & 64 \\
    & MLP layers                  & 3     & 3  & 3  & 3  \\
    \bottomrule
  \end{tabular}
 \begin{tablenotes}
    \item[*] Use the same planning parameters as GRADER.
 \end{tablenotes}
 \end{threeparttable}
\end{table}

\subsubsection{Model Structures and Hyper-parameters}
\label{app:implementation}

Since different nodes have different dimensions, we design a set of encoders $E_j$ and a set of decoders $D_j$ to convert the dimension of features. Thus, the entire model structure is 
\begin{equation}
    s_j^{t+1} = D_j (f_{\theta_j}(E_j([ \textbf{PA}_j^\gG]^t, N_j))),\ \ \ \forall j \in [M]
\end{equation}
We list all important hyper-parameters in the implementation for three environments in Table~\ref{app:parameters}.

\subsection{Broader Social Impact and Additional Limitation}
\label{app:social}

\subsubsection{Broader Social Impact}

We identify several important social impacts of our proposed method, including both positive and potential negative impacts:
\begin{itemize}[leftmargin=0.3in]
    \item[1)] Incorporating causality into reinforcement learning methods increases both the interpretability and generalizability of artificial intelligence, which helps users easily check the working progress of agents and the source of failures.
    \item[2)] Insufficient data and training may cause flawed causal graphs, which may lead to a wrong understanding of the causation of the task. This wrong understanding of the task may cause risky and irrational actions of agents.
    \item[3)] The discovered causal graph could be accessed and modified by users to manipulate the behaviors of agents on purpose. If the task contains private information, the discovered causal graph may cause privacy issues when the graph is interpreted by other users.
\end{itemize} 

To mitigate the potential negative societal impacts mentioned above, we encourage research to follow these instructions:
\begin{itemize}[leftmargin=0.3in]
    \item[1)] People should always check the convergence of the causal discovery step and verify the discovered causal graph with domain knowledge.
    \item[2)] The discovered causal graph should be frequently checked and verified with the training data to ensure its correctness. The causal graphs also need to be encrypted and only accessible to algorithms and trustworthy users.
\end{itemize}

\subsubsection{Additional Limitation}

\textbf{Causal discovery methods.} The gradient-based discovery method are widely investigated recently for large datasets since they have good scalability. However, these methods also require lots of training data to converge. In Online RL, we don’t have enough data at the beginning of training. Thus, constraint-based methods are more suitable for causal RL tasks.

Although our constraint-based causal discovery does not scale as well as the score-based methods, our proposed independence tests achieve a time complexity of $\Omega(|S|(|S|+|A|))$, which is tolerable for most RL problems with lower dimensional state space. Empirical studies also show that our independent tests enjoy better data efficiency. 

\textbf{Assumptions in our theoretical analysis.}
Faithfulness and Markov properties are commonly used in causal discovery literature such as~\cite{brouillard2020differentiable, chickering2002optimal}. It is claimed in \cite{chickering2002optimal} that the oracle independent test can be ensured by the satisfaction of Markov property and faithfulness. Recent work~\cite{pmlr-v119-addanki20a} also assumes the oracle of conditional independent test in its Assumption 2.1. Practically, the oracle test can be implemented with certain sub-linear sample complexity, as is investigated in~\cite{canonne2018testing}.

In reinforcement learning tasks, agents interact with the environment by doing interventions, which is achieved by assigning values to action nodes. Then, the intervention results are reflected by the states. Under fully observable Markov settings, the value of these states contains all information about the intervention. Thus, our RL setting usually satisfies the assumptions we use in the theoretical proof. 

\end{document}